\documentclass[letterpaper, 10 pt, journal, twoside]{IEEEtran} 

\usepackage{amsfonts}
\usepackage{amsmath}
\usepackage{amssymb}
\usepackage{mathtools}
\usepackage{tikz}

\usepackage{todonotes}          
\usepackage{comment}            

\newcommand{\R}{\mathbb{R}}     

\newcommand{\SO}{\text{SO}}     

\newcommand{\SE}{\text{SE}}

\DeclareMathOperator{\Ad}{Ad}   
\DeclareMathOperator{\ad}{ad}

\newcommand{\bts}{~\bar\times^*} 

\newcommand{\rmv}{{\rm v}}      
\newcommand{\rmf}{{\rm f}}      

\newcommand{\bI}{\mathbb{I}}     
\newcommand{\bM}{\mathbb{M}}     

\newcommand{\rms}{{\rm s}}      

\newcommand{\ls}{\hspace{0em}}   
                  
\newtheorem{theorem}{Theorem}[section]
\newtheorem{proposition}[theorem]{Proposition}

\newenvironment{proof}{
    \begin{trivlist}
    \item[]{\bf Proof.}
    }
    {
    \hfill{$\square$}
    \end{trivlist}
    }

\usetikzlibrary{decorations.text}

\begin{document}
%
\title{ 
On Centroidal Dynamics and \\ 
Integrability of Average Angular Velocity 
}

\author{ Alessandro Saccon, 
         Silvio Traversaro, 
         Francesco Nori, 
         Henk Nijmeijer }


%

\maketitle

\begin{abstract}
In the literature on robotics and multibody dynamics, 
the concept of average angular velocity has received 
considerable attention in recent years.
We address the question of whether the average 
angular velocity defines an orientation frame
that depends only on the current robot configuration 
and provide a simple algebraic condition to 
check whether this holds. 

In the language of geometric mechanics, 
this condition corresponds to requiring
the flatness of the mechanical connection 
associated to the robotic system. 
Here, however, we provide both 
a reinterpretation and a proof of this result 
accessible to readers with a background 
in rigid body kinematics and multibody dynamics
but not necessarily acquainted with differential geometry,
still providing precise links to the geometric mechanics literature.  

This should help spreading
the algebraic condition beyond the scope of geometric mechanics,
contributing to a proper 
utilization and understanding of the concept of 
average angular velocity.
\end{abstract}

\section{Introduction}

The total momentum of floating articulated robotic systems, 
such as aerial manipulators and humanoid robots, 
has received considerable attention in the robotic literature.
There is a growing consensus that the dynamics of total momentum 
can be used as a reduced but still exact model of the original system
that can ease, e.g., the development of posture and balance controllers
as well as planning algorithms for humanoid robots
\cite{wensing2016},
\cite{nava2016},
\cite{koolen2016design},
\cite{garofalo2015inertially},
\cite{dai2014whole},
\cite{ott2011posture},
\cite{lee2012momentum},
\cite{Kajita2003}.
The total momentum is defined as the sum of all linear and angular
momenta of the (rigid) bodies composing the articulated system.
The momentum is typically computed with respect to a frame
which orientation is that of the inertial frame 
and origin is the total center of mass \cite{Orin2013}.
Its time evolution depends only on the external forces and torques
acting on the system, such as gravity and contact forces.
The total angular momentum can be split into a linear and
an angular component. The linear component, when divided by
the total mass, captures the average linear velocity of the mechanism,
i.e, the velocity of the center of mass (CoM). 
Although still debating about which value it should 
be regulated to, the angular component 
has been used to define a concept of ``average angular velocity'' 
of the entire mechanism. The concept of average angular velocity is discussed in \cite{Orin2013} and
it corresponds, roughly speaking, 
to the angular velocity of entire mechanism,
for the given pose but assuming the internal joints to be fixed, 
corresponding to the current value of the angular momentum.

The geometric mechanics community has been employing a concept strictly
related to the average velocity, the locked velocity, for at least two decades 
\cite[Section 3.3]{marsden1992lecturemechanics}
\cite[Section 5.3]{bloch1996nonholonomic}, but an explicit link between the two concepts appears to be missing. We aim
at providing this link in this paper, hoping that this will help the communication of key results among different research communities
with a different theoretical background and research focus.

Paper's contributions:
{\bf (1) Equation of motions for a free-floating robot
written employing a robot-specific notation consistent with differential geometry notation}: this paper 
presents the dynamics of a simply supported 
articulated rigid-body system subject to 
external forcing by employing a notation which is inspired
by the spatial vector algebra notation \cite{featherstone2008, featherstone2008handbook} 
while allowing for a one-to-one mapping
with the concepts used in geometric mechanics and differential
geometry related to the theory of differentiable manifolds
and Lie groups \cite{Bloch2003,marsden2013introduction} (e.g., 
$X$ corresponds to $\Ad$,
$\times$ to $\ad$, and $\bts$ to $- \ad^*$,
see next section for details).
While the employment of Lie group
formalism is robotics is certainly not new
(see, e.g., the excellent publications
\cite{murray1994mathematical}
and
\cite{Park1995LieGroupFormulation}),
we felt that an explicit parallelism between 
spatial vector algebra 
and
Lie group notations was still missing.
This holds true, in particular, 
for the free-floating dynamics 
case treated here,
required for assessing the integrability 
of the average angular velocity;
{\bf (2) difference between the total momentum Noether's theorem and 
total momentum as commonly encounter in robotics}: we highlight how 
the total momentum considered in the robotic literature 
\cite{Orin2013} actually differs from the total momentum (the 
momentum map) that derives from the application of geometric 
mechanics version of Noether's theorem. As a consequence, 
the average velocity in the robotic literature and 
locked velocity in the geometric mechanics literature represent 
the same velocity, although expressed
with respect two different reference frames.
This apparently unessential detail plays however a key role 
(see discussion in Section~\ref{sec:lockedvel}) in understanding the 
main result of this paper. In highlighting the difference between the 
two velocities, the {\em angular momentum of the center of mass}, an 
extra component of the angular part of the total momentum 
once expressed with respect to the inertial frame, 
receives particular attention;
{\bf (3) Integrability condition for the average angular velocity}: We use the results and insights of the previous two points
to show that the fundamental question ``when does the integration of the average angular velocity define an orientation frame that depends only on the current internal joint position, independently of their time evolution?'' is equivalent to asking if a series of vector valued functions can be interpreted as the (right trivialized) partial derivatives of a nonlinear function of the internal joints.
While pointing out that this is {\em equivalent} to requiring, in terms of geometric mechanics, that the associated mechanical connection is flat (see, e.g., the discussion
on holonomy in \cite[Section 3.14]{Bloch2003}
and reference there in), we provide a reinterpretation of the result, and the associated algebraic condition ensuring the flatness of the connection, that can be easily followed by a reader 
acquainted with multibody dynamics, kinematics of rigid transformations, undergraduate calculus, and the understanding of Schwartz's theorem (symmetry of second derivatives), but with no or limited experience with differential geometry. 
Our hope is to contribute to the diffusion and utilization of this algebraic condition beyond the scope of geometric mechanics.

Rigid body notation is reviewed in Section~\ref{sec:notation}. Section~\ref{sec:momentum} reviews the dynamics of floating mechanical systems and the evolution of the total momentum. 
In Section~\ref{sec:centroidal}, we review 
the concepts of average and locked velocities,
centroidal frame, and provide the algebraic condition to show
when it depends only of the current pose and shape of the mechanism.
Conclusions and a discussion are provided in Section~\ref{sec:conclusion}.

\section{notation} \label{sec:notation}

This section introduces basic notation used for dynamics computations. We refer to \cite{eindhovenNotation}
for details.

Given a vector $w = (x,y,z)^T \in \R^3$, $w^\wedge$ (read $w$ {\em hat}) is the $3\times3$ \emph{skew-symmetric} matrix associated with the cross product $\times$ in $\R^3$, so that $w^\wedge x = w \times x$.
Given the \emph{skew-symmetric matrix} $W = w^\wedge$, 
$W^\vee \in \mathbb{R}^3$ (read $W$ {\em vee}) denotes
the inverse transformation. The set of rotational
matrices is denoted $\SO(3)$, the set of rigid transformations
$\SE(3)$. An element of $\SE(3)$ has the structure $[R, ~ o~;~ 0_{1\times 3}, 1] \in \R^{4\times 4}$, with $R \in \SO(3)$,
$o \in \R^3$, where ; denotes row concatenation.

\subsection{Frames notation} 
A {\em frame} is defined by a point, called {\em origin}, and an {\em orientation frame} \cite{de2013geometric}. We use 
capital letters to indicate frames. Given a frame $F$, we denote with $o_F$ its origin and with $[F]$ its orientation frame, writing $F = (o_F, [F])$.\\[.3ex]
 
\begin{tabular}{@{}p{0.08\textwidth}p{0.33\textwidth}@{}}
$A$, $B$, $\dots$   & Reference frames         \\
$p$                 & An arbitrary point        \\
$B[A]$              & Frame with origin $o_B$ 
                      and orientation $[A]$   
\end{tabular}

\subsection{Coordinates vectors and transformation matrices notation} 
\noindent
\begin{tabular}{@{}p{0.23\textwidth}p{0.22\textwidth}@{}}
$\ls^Ap \in \R^3$             & {\small Coordinates of $p$ 
                                w.r.t. $A$}
\\
$\ls^Ao_B \in \R^3$           & {\small 
 								Coordinates of $o_B$ 
                                w.r.t. to $A$
                                }
\\
$\ls^AR_B \in \R^{3\times3}$  & {\small 
								Rotation matrix 
                                from $[B]$ to $[A]$ 
                                }
\\
$\ls^A H_B = 
\begin{bsmallmatrix} 
\ls^A R_B & \ls^A o_B \\
0_{1 \times 3} & 1 
\end{bsmallmatrix}$          & {\small 
								Rigid transf.
                                from $B$ to $A$
                                }
\\[1ex]
$\ls^AX_B \hspace{.2ex}
= 
\begin{bsmallmatrix}
 \ls^A R_B     &           \ls^A o^{\wedge}_B  \ls^A R_B \\
  0_{3\times3} &  \phantom{\ls^A o^{\wedge}_B} \ls^A R_B 
\end{bsmallmatrix}$          
                              & {\small 
                              	Velocity transf.
                                from $B$ to $A$ 
                                }
\\[1ex]
$\ls^C\rmv_{A,B} 
=
\begin{bsmallmatrix}
  \ls^C v_{A,B} \\
  \ls^C \omega_{A,B} 
\end{bsmallmatrix}
\in \R^{6}$                  & {\small 
							   Velocity of $B$ w.r.t. to $A$
                               expressed in $C$ 
                               }
\end{tabular}
\begin{tabular}{@{}p{0.23\textwidth}p{0.22\textwidth}@{}}
$
\ls^B\rmv_{A,B} \times = 
\begin{bsmallmatrix}
\ls^B\omega_{A,B}^\wedge & 
\ls^Bv_{A,B}^\wedge \\
0_{3 \times 3} & 
\ls^B\omega_{A,B}^\wedge
\end{bsmallmatrix}$           & {\small 
								 Vector cross product in $\R^6$ 
                                 }
\\[1ex]
$\ls_AX^B = \begin{bsmallmatrix}
  \phantom{\ls^A o^{\wedge}_B} \ls^A R_B & 0_{3\times3} \\
           \ls^A o^{\wedge}_B  \ls^A R_B & \ls^A R_B \end{bsmallmatrix}$          & {\small 
             					Wrench transf.
                                from $B$ to $A$ 
                                ($\ls_AX^B = \ls^AX_B^{-T}$) 
                                }
\\ 
$\ls_B\rmf =
\begin{bsmallmatrix}
  \ls^C f_{A,B} \\
  \ls^C \tau_{A,B} 
\end{bsmallmatrix}
\in \R^{6}$                   & {\small 
								Coordinates of the wrench 
                                $\rmf$ w.r.t. $B$
                                }
\\
$
\ls^B\rmv_{A,B} \bar\times^* =
\begin{bsmallmatrix}
\ls^B\omega_{A,B}^\wedge & 
0_{3 \times 3} \\
\ls^Bv_{A,B}^\wedge & 
\ls^B\omega_{A,B}^\wedge
\end{bsmallmatrix}$             & {\small 
								   Dual cross product 
                                   in $\R^6$ 
                                   }
\end{tabular}
\begin{tabular}{@{}p{0.23\textwidth}p{0.22\textwidth}@{}}
$\ls_{B} \bM_B = 
\begin{bsmallmatrix}
      m 1_{3 \times 3}       &  m \, \ls^B c^{\wedge} \\ 
     -m \, \ls^B c^{\wedge}  &  \ls_B \bI_B
  \end{bsmallmatrix}
$                                & {\small 
								   Generalized inertia matrix
                                   w.r.t. frame $B$
                                   }
\end{tabular}\vspace{1ex}
In the expression for $\ls_{B} \bM_B$ 
(where $B$ is typically a body fixed frame), 
$m$ is the body mass, $\ls^B c$ the CoM coordinates, 
and $\ls_B \bI_B$ the rotational inertia w.r.t. $B$.
We use
$$\ls^D J_{A,C/B}$$ to
indicate the Jacobian relating the velocity 
of frame $C$ with respect to $A$
expressed in $D$ with
the velocity of the base link expressed in $B$, so that
$\ls^D \rmv_{A,C} = \ls^D J_{A,C/B}(q) \, \nu$ with 
$\nu = (\ls^B \rmv_{A,B}, \dot q_J)$.

\subsection{Frame velocity representation}
\label{subsec:bodyVel}
The velocity of $B$ w.r.t. $A$ is given by $\ls^A \dot{H}_B$, however, it is more common to express it as a one the following six-dimensional vectors,
\begin{align*}
  \ls^A \rmv_{A,B} 
& = 
  \begin{bmatrix} 
    \ls^A {\dot o}_B 
    - \ls^A \omega_{A,B}^\wedge \ls^A o_B 
  \\ 
    \ls^A \omega_{A,B}
  \end{bmatrix} ,
\\
  \ls^B \rmv_{A,B} 
  & = 
  \begin{bmatrix} 
    \ls^B R_A \ls^A \dot{o}_B  \\ 
    \ls^B R_A \ls^A \omega_{A,B}
  \end{bmatrix}, 
  & 
  \ls^{B[A]} \rmv_{A,B} 
  & =  
  \begin{bmatrix} 
    \ls^A \dot{o}_B \\ 
    \ls^A \omega_{A,B}
  \end{bmatrix} ,
\end{align*}
where $\ls^A \omega_{A,B} := (\ls^{A} \dot{R}_{B} \ls^A R_B^\top)^\vee$.
We refer to $\ls^A \rmv_{A,B}$, $\ls^B \rmv_{A,B}$,
and $\ls^{B[A]} \rmv_{A,B}$ as, respectively,
the {\em right-trivialized}, the {\em left-trivialized},
and the {\em mixed} velocity of $B$ w.r.t. $A$. 
The mixed representation is also known as \emph{hybrid} representation~\cite{bruyninckx1996symbolic}. 

The left- and right-trivialized representations are widespread in the literature of Lie group-based geometric mechanics \cite{murray1994mathematical} 
(where they are called \emph{spatial} and \emph{body} velocities) 
and recursive robot dynamics algorithms \cite{featherstone2008,jain2010robot} (where they are called \emph{spatial} velocities). The mixed velocity is commonly used in multi-task control frameworks \cite{siciliano2010robotics,chiaverini2008kinematically,nava2016}.

\subsection{Single Rigid Body Dynamics}
Given a rigid body whose position in space is
determined by $\ls^A H_L$ with $L$ fixed to the body
the classical Newton-Euler equations are written,
in a combined form, as
\begin{align}\label{eq:ELsingleBody}
 \ls_L \bM_L \ls^L \dot \rmv_{A,L} +
 \ls^L \rmv_{A,L} \bts  \ls_L \bM_L \ls^L \rmv_{A,L}  
& = 
  \ls_L \rmf
\end{align}
with $\ls_L \rmf$ denoting the external wrench (combined force and torque vector) expressed w.r.t. $L$ and $\bts$ denotes the dual 6D cross product ($\ls^L \rmv_{A,L} \bts$ is equivalent
to $-\ad^*_{\ls^L \rmv_{A,L}}$ in the language of Lie groups). 
We use the letter L since a rigid body on an articulated 
mechanism is usually referred to as a {\em link}.

\section{Rigid-body dynamics}\label{sec:momentum}

\subsection{Floating systems with gravity and external contact forces}

Consider a robotic system whose configuration is given 
by $q = (H, \rms) \in \SE(3) \times \R^{n_J}$,
where $H = \ls^A H_B$ denotes the base link's homogeneous transformation
matrix and $\rms$ represents the displacement 
of the $n_J$ internal joints. We will refer to
$H$ as the {\em pose} and to $\rms$ as the {\em shape} 
of the robot.
The velocity of the mechanism is
parameterized via $\nu = (\rmv, \dot \rms) \in \R^6 \times \R^{n_J}$ with $\rmv = \ls^B \rmv_{A,B}$ 
denoting the velocity of the base frame with respect the
inertial frame expressed in the base frame 
($\ls^B \rmv_{A,B}$ is a {\em left trivialized} velocity, cf. the notation section).

The dynamics of a floating articulated robotic system such as, e.g., a humanoid robot
\cite{wieber2015handbook} is usually written as
\begin{align}\label{eq:forcedstandardEL}
  M(q) \dot\nu + C(q, \nu) \nu + G(q)
  & = 
  [0 \,; \tau] + \sum_i (\ls^i J)^T \, \ls_i \rmf 
\end{align}
where $M$, $C$, and $G$ are, respectively, the mass matrix, Coriolis matrix, and potential force vector, $\tau$ is the internal joint torques, and  $\ls^i J$ and $\ls_i \rmf$ are the $i$-th contact Jacobian and contact force, both expressed
with respect to a contact frame $C_i$, 
fixed with respect to its corresponding link. 
To derive \eqref{eq:forcedstandardEL}, 
one can take a Newton-Euler approach 
summing up all the contributions of the internal and
external forces for each body using \eqref{eq:ELsingleBody}
or setting up a Lagrangian $L(H, s, \dot H, \dot s)$ and employ the Euler-Lagrange equations. However, as $H$ is not a vector quantity, 
either one use a local vector parametrization for $H$ (based, e.g., on Euler angles) or employ the tools from geometric mechanics and form the left trivialized Lagrangian
\begin{align}\label{eq:lt_lagrangian}
l( H, \rms, \rmv, \dot \rms) = 
\frac 12 
\begin{bmatrix}
  \rmv \\
  \dot \rms
\end{bmatrix}^T
\begin{bmatrix}
  \mathbb{L}(\rms)   & \mathbb{A}(\rms) \\
  \mathbb{A}^T(\rms) & \mathbb{S}(\rms) 
\end{bmatrix}
\begin{bmatrix}
  \rmv \\
  \dot \rms
\end{bmatrix}
\end{align}
where $\rmv = \ls^B \rmv_{A,B}$ satisfies 
$\dot H = H \rmv^\wedge$ and
$\mathbb{L}$, $\mathbb{A}$, and $\mathbb{S}$
are suitable partitions of the overall mass matrix $M(q)$
appearing in \eqref{eq:forcedstandardEL}
in accordance with the dimension of $\rmv$ and $\dot s$.
The matrix $\mathbb{L}$ is typically referred to
as the {\em locked inertia tensor} as is corresponds to the
(generalized) inertia of the entire mechanism 
computed with respect to $B$ assuming no motions of its internal joints.
To be more precise, one should write
$\mathbb{L}$ as $\ls_B\mathbb{L}_B$ and
$\mathbb{A}$ as $\ls_B\mathbb{A}$ 
to indicate the output (and, for $\mathbb{L}$, also input) frame 
that these transformations accept. This should 
help, e.g., to better interpret the
expression for the combined linear and angular
momentum given by \eqref{eq:momentum},
later in the text.

The dynamics of the articulated mechanism can 
be then derived using Hamel equations 
(see, e.g., 
\cite{marsden1993reduced},
\cite[Section 13.6]{marsden2013introduction}, \cite{Ball2012}), namely
\begin{align}\label{eq:NE}
\frac{d}{dt} \frac{\partial l}{\partial  \rmv} 
  + \rmv \bts \frac{\partial l}{\partial \rmv} 
  & = 0 
\\ \label{eq:EL}
\frac{d}{dt} \frac{\partial l}{\partial \dot \rms} 
  - \frac{\partial l}{\partial \rms} & = 0. 
\end{align}
Hamel equations are a combination of standard 
Euler-Lagrange equations \eqref{eq:EL}
and
Newton-Euler equations \eqref{eq:NE}, 
the latter also called Euler-Poincar\'e equations for a generic Lie group
\cite[Section 13.5]{marsden2013introduction}.
In the presence of internal and external forces 
and the presence of potential energy due to, e.g., 
the effect of gravity, 
\eqref{eq:NE}-\eqref{eq:EL} become
\begin{align}\label{eq:forcedNE}
\frac{d}{dt} \frac{\partial l}{\partial  \rmv}  
  + \rmv \bts \frac{\partial l}{\partial \rmv}
  & = H^{-1} \frac{\partial l}{\partial H} 
    + \sum_i (\ls^i X)^T \, \ls_i \rmf \\
 \label{eq:forcedEL}   
\frac{d}{dt} \frac{\partial l}{\partial \dot \rms} 
  - \frac{\partial l}{\partial \rms} 
  & =  \tau + \sum_i (\ls^i S)^T \, \ls_i \rmf ,
\end{align}
where $\tau$ and $\ls_i \rmf$ are as in \eqref{eq:forcedstandardEL},
$H^{-1} {\partial l}/{\partial H}$ 
the vector representation of the linear map
$w \mapsto D_1 l(H,\rms,\rmv,\dot\rms) \cdot H w^\wedge$, $w \in \R^6$, and 
$\ls^i X$ and $\ls^i S$ define, respectively, 
the pose and shape parts of the $i$-th contact Jacobian $\ls^i J$. 
More precisely, the (mixed) velocity of the $i$-th contact point
satisfies 
\begin{align}\label{eq:JacobianStructure}
  ^i \rmv
  & = \ls^i J \, \nu 
    = \ls^i X \rmv + \ls^i S \, \dot \rms 
\end{align} 
with
$^i \rmv := \ls^{C_i[A]} \rmv_{A,C_i}$ 
and 
$\ls^i J := \ls^{C_i[A]} J_{A,C_i/B}$
the $i$-th contact Jacobian (we refer to Section~\ref{sec:notation} for a clarification on the notation $\ls^{D} J _{A,C/B}$). Note that, 
by definition, $\ls^i X = \ls^{C_i[A]} X_B$, implying 
$\ls^i X^T = \ls_B X^{C_i[A]}$, a wrench transformation. 
Except for the notation,
\eqref{eq:forcedNE}-\eqref{eq:forcedEL} are 
equivalent to the more common \eqref{eq:forcedstandardEL}
but they provide extra structure that 
helps to understand the definition and time evolution of the total momentum
directly from the equations of motion.
From a computational point of view,
forward and inverse dynamics for \eqref{eq:forcedstandardEL} 
can be obtained using, e.g., 
the floating-base recursive Newton-Euler algorithm and composite-rigid-body algorithm 
presented in \cite{featherstone2008}.
 

\subsection{The total momentum expressed in the inertial frame}

The Lagrangian \eqref{eq:lt_lagrangian}
is not a function of $H$ meaning it is invariant 
with respect to a rigid transformation.
As shown in Appendix~\ref{sec:AppendixA}, standard 
results of geometric mechanics imply that the quantity 
\begin{align}\label{eq:momentum} 
  \ls_A \mathcal{J}
  & = 
  \ls_A X^B
  \left(  
     \mathbb{L} \, \rmv
     + 
     \mathbb{A} \, \dot\rms
  \right)
\end{align}
is a constant of motion for the unforced system. 
In \eqref{eq:momentum},
$\mathbb{L}$ and $\mathbb{A}$ are as in \eqref{eq:lt_lagrangian}
and $\ls_A X^B$ as in Section~\ref{sec:notation} with
$A$ denoting the inertial frame and $B$ the base link frame.

Recalling that \eqref{eq:lt_lagrangian} is obtained
by summing up all kinetic energies of each link dynamics
\eqref{eq:ELsingleBody} employing link Jacobian with a structure similar to \eqref{eq:JacobianStructure},
it is straightforward to recognize in $\ls_A \mathcal{J}$ 
the {\em total momentum} given by the sum of the all the linear
and angular momenta of each rigid body expressed with respect to the origin of $A$. 
When gravity and external forcing are present, $\ls_A \mathcal{J}$
evolves according to
\begin{align}\label{eq:forcedMomentum}
  \frac{d}{dt} \ls_A \mathcal{J} 
  & = 
  \ls_A X^B \, 
  \left( 
    H^{-1} \frac{\partial l}{\partial H} 
    + \sum_i  (\ls^i X)^T  \, \ls_i \rmf 
  \right) .
\end{align}
This result can be derived directly 
from a straightforward modification of Noether's theorem (for a proof of
the geometric version of Noether's theorem, see \cite[Chapter 3]{Bloch2003}). One careful inspection of the above formula shows, however, that \eqref{eq:forcedMomentum} is actually equivalent to \eqref{eq:forcedNE}, only written in the inertial frame~$A$. 

\subsection{The total momentum expressed at the center of mass}

The momentum of the system can be expressed also 
with respect to other frames. In particular, 
the frame $G := (p_{com}, [A])$,
that has as origin the combined CoM $p_{com}$ 
and the orientation of the inertial frame $A$,
is commonly found in the robotic literature \cite{orin2008centroidal,Orin2013}. 
With respect to $G$, the momentum is given by
\begin{align}
  \ls_G \mathcal{J}(H,\rms,\rmv, \dot\rmv) 
  & := \ls_G X^A(H,\rms) \, \ls_A \mathcal{J}(H,\rms,\rmv, \dot\rmv) .
\end{align} 
We refer to $\ls_G \mathcal{J}$ as the {\em centroidal momentum} (in accordance with, e.g., \cite{Orin2013}). Remarkably,
when no external forces and potential are present, 
this quantity is also constant as $\ls_A \mathcal{J}$ given in \eqref{eq:momentum} is constant and because $(\ls_G {\dot X}^A ) \ls_A \mathcal{J}$ is {\em always} a zero.  
This last fact is related to the angular momentum of the CoM. 
Denoting with $m$ the total mass, the angular momentum of the CoM is simply given by $\ls^A p_{com} \times m \, \ls^A {\dot p}_{com}$. 
The only difference between $\ls_G \mathcal{J}$ and $\ls_A \mathcal{J}$ is indeed that $\ls_A \mathcal{J}$ contains within its angular part (i.e., its last three elements) also the angular momentum of the CoM. Then
$(\ls_G {\dot X}^A ) \ls_A \mathcal{J} \equiv 0$ as
$\ls^A {\dot p}_{com} \times m \, \ls^A {\dot p}_{com} 
\equiv 0$. This also implies that replacing $A$ with $G$
in \eqref{eq:forcedMomentum} is all we need to obtain
the evolution of $\ls_G \mathcal{J}$.

\subsection{Locked and average velocities}
\label{sec:lockedvel}

In geometric mechanics, a special role is played by the left-trivialized locked velocity 
$\ls^B \rmv_{loc}$ defined in such a way that 
\begin{align}
  \ls_B \mathcal{J} 
  & = 
  \mathbb{L}(s) \rmv  + \mathbb{A}(s) \dot \rms 
  =:  
  \mathbb{L}(s) \, \ls^B \rmv_{loc}
\end{align}
or, equivalently,  
\begin{align}\label{eq:defvloc}
 \ls^B \rmv_{loc}  := \rmv  + \mathbb{L}^{-1}(s) \mathbb{A}(s) \dot \rms . 
\end{align}
The locked velocity is, for each instant of time,
the velocity at which the base link 
should move, while considering the internal joints as locked,
to get the same value of the momentum corresponding to the current velocity of the mechanism. Note how 
the locked velocity is expressed with respect to
the base frame, so it is only a function 
of $\rmv$, $\rms$, and $\dot\rms$. 

In the robotic literature, the {\em average velocity}
is defined as $\ls^G \rmv_{ave} := \ls^G X_B \ls^B \rmv_{loc}$
\cite{Orin2013}. Note that we can obtain
$\ls^G \rmv_{ave}$ from
\begin{align}
  \ls^G \rmv_{ave} 
  & = 
  \ls_G^{~} \mathbb{L}_G^{-1} \, \ls_G^{~} \mathcal{J}
\end{align}
with ${\ls_G \mathbb{L}_G} := \ls_G X^B {\ls_B \mathbb{L}_B} \ls^B X_G$ block diagonal. 
At this point is key to observe that 
$\ls^G \rmv_{ave}$ depends also on $H$,
other than $\rmv$, $\rms$, and $\dot\rms$
while $\ls^B \rmv_{loc}$ does not.
We will show in the next section, 
that the latter is then preferable 
in answering the integrability question 
for the average (angular) velocity. 

As a side note, it is worth recalling 
at this point that both
the locked velocity and average velocity
can be used to block diagonalize the mass matrix.
However, as the kinetic energy is now written
with respect to another velocity one
must be careful in deriving the equations of motion
from it. When using the locked velocity, 
Lagrange-Poincar\'e equations 
(see, e.g., \cite[Chapter 13]{marsden2013introduction})
can be use to retrive a set of equations
equivalent to \eqref{eq:NE}-\eqref{eq:EL}.


\section{The centroidal frame and the main result}\label{sec:centroidal}

In this section, we recall the concept of centroidal frame
and provide the algebraic condition ensuring
that it depends only on the configuration $(H,s)$.

\subsection{The centroidal frame and the main question}
\label{sec:centroidalFrame}

The definition of the locked velocity given by \eqref{eq:defvloc}
allows one to write the momentum with respect to $A$ simply as 
\begin{align}
  \ls_A \mathcal{J} 
    & = 
  \ls_A X^B \ls_B \mathbb{L}_B \, \ls^B\rmv_{loc} .
\end{align}
Posing 
$\ls_A \mathbb{L}_A := \ls_A X^B \ls_B \mathbb{L}_B \ls^B X_A$ 
and $\ls^A\rmv_{loc} := \ls^A X_B \ls^B\rmv_{loc}$,
the equation above simply becomes
\begin{align}
  \ls_A \mathcal{J} 
    & = 
   \ls_A \mathbb{L}_A \, \ls^A\rmv_{loc} . 
\end{align}
For a given initial configuration $\ls^A H_C(t_0) \in \SE(3)$ 
at a given time $t = t_0$, one can then integrate the differential equation
\begin{align} \label{eq:centroidalFrame_ODE}
  \ls^A {\dot H}_C = \ls^A \rmv_{loc}^\wedge \, \ls^A H_C 
\end{align}
to get a frame that has, as right trivialized velocity, the locked velocity $\ls^A \rmv_{loc}$. A key remark that keeps appearing in the robotic literature \cite{orin2008centroidal,Orin2013,ding2016dynamic}
is that the solution $\ls^AH_C$ of \eqref{eq:centroidalFrame_ODE}
is not guaranteed to depend only on the configuration $(H, s)$. 
That is, when $\ls^AH_C$ satisfies $\ls^AH_C(t) = \ls^A F(H(t), s(t))$
for a suitable function $\ls^A F: \SE(3) \times \R^{n_J} \rightarrow \SE(3)$.
It is well known that this does not always happen as in the simulation results presented in \cite{Wieber2005} and
references therein. What appear to
be less known is that there is 
a simple condition to check when this
happens and that this is related to
asking if the columns of $\mathbb{L}^{-1} \mathbb{A}$ are partial derivatives of 
a nonlinear function of the joint displacements.

Note that the initial condition for $C$ at $t=t_0$ is completely arbitrary and therefore the function $\ls^A F(H, s)$, where it exists, is determined up to an arbitrary {\em right} multiplication by an element of $SE(3)$. Furthermore, given the fact that the kinematics \eqref{eq:centroidalFrame_ODE} is actually
invariant to an arbitrary pose transformation, one gets the extra condition that if $\ls^A F$ exists,
it must be of the form $\ls^A F(H,s) = H \, \ls^B F(s)$
with $\ls^B F(s): \R^{n_J} \rightarrow \SE(3)$. The dependence of the centroidal frame only on the configuration is therefore equivalent to the existence of this $\SE(3)$-valued function of the internal joints.
Tipically, the frame $C$ at $t=t_0$ is taken to have 
its origin coinciding with the total center of mass because
it can be shown (see remark below) that \eqref{eq:centroidalFrame_ODE} will maintain the equivalence of the two points: this justifies
the use of centroidal frame as name for $C$.

{\bf Remark.} Independently of the existence of a configuration-dependent-only frame satisfying $\ls^AH_C(t) = \ls^A F(H(t), s(t))$,
the CoM $p_{com}$ has always constant coordinates with respect to a frame $C$ that evolves according to \eqref{eq:centroidalFrame_ODE}. 
In formulas, $\ls^C \dot p_{com} = d/dt \, \ls^C p_{com} \equiv 0$. A proof of this fact is given in the Appendix~\ref{sec:AppendixB}. Therefore, in case indeed we can find a function $\ls^A F$ such that $\ls^AH_C(t) = \ls^A F(H(t), s(t))$, it then seems natural to choose the frame $C$ such that $p_{com}$ is its origin. 
\hfill$\blacksquare$

\subsection{The main result}

Define
\begin{align}
   \mathcal{A}(s) := \mathbb{L}^{-1}(s) \mathbb{A}(s)
\end{align}
with $\mathbb{L}(s)$ and $\mathbb{A}(s)$ 
expressed with respect to $B$ as in \eqref{eq:lt_lagrangian}. Then, the following holds.

\begin{proposition} \label{prop:integrability}
The centroidal frame satisfying \eqref{eq:centroidalFrame_ODE}
is integrable, that is, there exist a (differentiable) function
$\ls^B F : \R^{n_J} \rightarrow \SE(3)$
such that
\begin{align} \label{eq:integrability}
  \ls^AH_C(t) = H(t) \, \ls^B F(s(t))
\end{align}\label{eq:closure}
if and only if 
\begin{align} \label{eq:curvatureconnectionSE3}
  \mathcal{B}_{ij}
& :=
  \frac{\partial \mathcal{A}_i}{\partial s_j} -
  \frac{\partial \mathcal{A}_j}{\partial s_i} +
  \mathcal{A}_i \times \mathcal{A}_j \equiv 0 
\end{align}
for every $i,j \in \{1,2,\dots,n_J\}$, where
$\mathcal{A}_i$, $\mathcal{A}_j$ are the i-th and $j$-th columns
of $\mathcal{A}$ and $\times$ the vector cross product
in $\mathbb{R}^6$.\hfill$\square$
\end{proposition}

\begin{proof} The result is, roughly speaking,
a generalization to $\SE(3)$
of Schwartz's theorem (symmetry of second derivatives)
and the related theorem stating that closed 
differential forms are locally exact 
(existence of a potential function).
We start by the observation that
if we take an arbitrary sufficiently smooth function $F: \R^{n_J} \rightarrow \SE(3)$ the right trivialized velocity associated to the homogeneous transformation matrix 
\begin{align}
  H \, F(s) \in \SE(3) ,
\end{align}
with $H = H(t) \in \SE(3)$ and $\rms = \rms(t) \in \R^{n_J}$ 
sufficiently smooth curves is equal to 
\begin{align} \label{eq:AXB_vpAsdots}
  \left(\frac{d}{dt}(HF) (HF)^{-1} \right)^\vee =
  \ls^A X_B \, ( \, \rmv + \mathcal{A}(s) \dot \rms\, ) .
\end{align}
The above formula has exactly the same structure of
the locked velocity given in \eqref{eq:defvloc} 
with the extra interpretation that the columns of 
the matrix $\mathcal{A}(s)$ above satisfy,
for $i \in \{1,2,\dots,n_J\}$, 
\begin{align}\label{eq:trivializedDerivatives}
  \mathcal{A}_i(s) 
& = 
  \left(
  \frac{\partial F}{\partial s_i}(s) F(s)^{-1} 
  \right)^{\vee} \in \R^6 ,
\end{align}
that is, $\mathcal{A}_i$ can be
interpreted as the right trivialized partial derivatives of 
the nonlinear function $F$ depending on the internal joints.
Thinking $F$ as a function $F: \R^{n_J} \rightarrow \R^{4 \times 4}$ 
one knows that its second derivative must satisfy
\begin{align}\label{eq:schwartz}
  \frac{\partial^2 F}{\partial s_j \partial s_i} 
& = 
  \frac{\partial^2 F}{\partial s_i \partial s_j}  
\end{align}
for every $i$ and $j$. Writing 
$\partial F(s) / \partial s_i = \mathcal{A}^{\wedge}_i(s)  F(s)$
using \eqref{eq:trivializedDerivatives}, one can write
the left hand side of \eqref{eq:schwartz} as
\begin{align}\label{eq:secondDerF}
  \left( \frac{\partial \mathcal{A}_i}{\partial s_j} \right)^{\wedge}(s) F(s)
  +
  \mathcal{A}^{\wedge}_i(s) \mathcal{A}^{\wedge}_j(s) F(s)
\end{align}
and similarly for the right hand side by exchanging $i$ with $j$.
Equating the two expressions and multiplying on the right by the inverse of $F(s)$, one obtains straightforwardly
\eqref{eq:curvatureconnectionSE3} 
by recalling Jacobi identity 
$
  \mathcal{A}^{\wedge}_i \mathcal{A}^{\wedge}_j
  -
  \mathcal{A}^{\wedge}_j \mathcal{A}^{\wedge}_i
  =
  (\mathcal{A}_i \times \mathcal{A}_j)^{\wedge}
$. This proves the necessity of the condition.
Sufficiency is provided constructively, 
using
\begin{align}\label{eq:defF}
  F(s) 
  & = 
  \Delta_{n_J}(s_1, s_2, \dots, s_{n_J}) 
\end{align}
where $\Delta_{n_J}$ is given below and 
showing that $F$ in \eqref{eq:defF} satisfies
\eqref{eq:trivializedDerivatives} as long as
\eqref{eq:curvatureconnectionSE3} holds.
In \eqref{eq:defF}, $\Delta_{n_J}$ 
and the functions $\Delta_{i}: \mathbb{R}^i \rightarrow \SE(3)$, $i \in \{1,2,\dots,n_J\}$
are defined recursively 
for a given value $s = (s_1, \dots, s_{n_J})$
as 
the solution at $\sigma = s_i$ of the matrix differential equation 
\begin{align}\label{eq:Deltai}
\frac{d}{d\sigma} \Delta_i 
& =
  \mathcal{A}_i^\wedge(s_1, \dots, s_{i-1}, \sigma, ..., 0) \, 
  \Delta_i
\end{align}
with $\Delta_i$ evaluated at $(s_1, \dots, s_{i-1}, \sigma) \in \mathbb{R}^i$  with initial condition
\begin{align}\label{eq:Deltai0}
\Delta_{i}(s_1, \dots, s_{i-1}, 0) =
\Delta_{{i-1}}(s_1, \dots, s_{i-1}) 
\end{align}
and 
$\Delta_{1}(0) = F(0) \in \SE(3)$ arbitrary (the desired orientation of the centroidal frame for $s = 0$). Note that 
$d\Delta_i/d\sigma $ in \eqref{eq:Deltai} equals
$ \partial \Delta_i (s_1, \dots, s_{i-1},\sigma)/ \partial s_i$.
In the following, we provide the proof that \eqref{eq:trivializedDerivatives}
holds for $n_J = 2$. Proving \eqref{eq:trivializedDerivatives} 
for $n_J > 2$ is a straightforward but tedious calculation
that follows from the technique used when $n_J = 2$. Given $F(s) := \Delta_2(s_1, s_2)$, it is immediate to see one must have
$\partial F / \partial s_2 (s) = \mathcal{A}^\wedge
_2(s) F(s) $. Showing $\partial F / \partial s_1 (s) = \mathcal{A}^\wedge_1(s) F(s) $ is, instead, more involved
and requires \eqref{eq:curvatureconnectionSE3}.
From \eqref{eq:defF}, \eqref{eq:Deltai}, and \eqref{eq:Deltai0}, 
for the case $n_J = 2$, one
can write $\partial F / \partial s_1 (s)$ as
\begin{align}\label{eq:suffstep1}
  \int_0^{s_2} 
    \frac{\partial}{\partial s_1} 
    \left(
      \mathcal{A}_2^\wedge(s_1, t) \Delta_2(s_1,t) 
    \right)
  dt
  + 
  \frac {\partial \Delta_2}{\partial s_1}(s_1,0) ,
\end{align}
that can be further expanded into
\begin{align}\label{eq:suffstep2}
  \int_0^{s_2} 
    \frac{\partial \mathcal{A}_2^\wedge}{\partial s_1}
    \Delta_2 
    +
    \mathcal{A}_2^\wedge
    \frac{\partial \Delta_2}{\partial s_1}
  dt
  + 
  \mathcal{A}_1^\wedge(s_1, 0)
  \Delta_2(s_1, 0) ,
\end{align}
where the terms inside the integral are evaluated at
$(s_1, t)$. Using \eqref{eq:curvatureconnectionSE3}, one gets
that the integral above can be written as
\begin{align}\label{eq:suffstep3}
  \int_0^{s_2} 
    \left( 
      \frac{\partial \mathcal{A}_1^\wedge}{\partial t}
      +
      \left( \mathcal{A}_1 \times \mathcal{A}_2 \right) ^\wedge
    \right)  
    \Delta_2 
    +    
    \mathcal{A}_2^\wedge
    \frac{\partial \Delta_2}{\partial s_1}
  dt
\end{align}
Using again the Jacobi identity and recalling that ${\partial \Delta_2}/{\partial t} = \mathcal{A}_2^\wedge \Delta_2$, the above integral can be rewritten as
\begin{align}\label{eq:suffstep4}
  \int_0^{s_2} 
    \frac \partial {\partial t}
    \left(
      {\mathcal{A}_1^\wedge}
      \Delta_2
    \right)
    + 
    \mathcal{A}_2^\wedge
    \left(
       \frac{\partial \Delta_2}{\partial s_1}
       -
       \mathcal{A}_1^\wedge \Delta_2
    \right)
  dt
\end{align}
where all terms are evaluated at $(s_1,t)$. 
By assuming
${\partial F}/{\partial s_1}(s) = {\partial \Delta_2}/{\partial s_1}(s) = \mathcal{A}_1^\wedge(s) \Delta_2(s)$, the second term in the integral
of \eqref{eq:suffstep4} vanishes and allowing one to rewrite \eqref{eq:suffstep2} as 
\begin{align}
  \int_0^{s_2} 
    \frac \partial {\partial t}
    \left(
      {\mathcal{A}_1^\wedge}(s_1,t)
      \Delta_2(s_1,t)
    \right)
  dt
  +
  \mathcal{A}_1^\wedge(s_1, 0) \Delta_2(s_1, 0) 
\end{align}
which equals $\mathcal{A}_1^\wedge(s) \Delta_2(s) = \mathcal{A}_1^\wedge(s) F(s)$, with no contradiction on the assumption. 

\end{proof}

{\bf Remark.}
When the underlying Lie group is $\SE(3)$,
the expression given in \eqref{eq:curvatureconnectionSE3}
is equivalent to the curvature of a principal connection (see \cite[Chapter 2]{Bloch2003} and reference therein 
for a concise and convenient summary of principal connections).
The curvature of a principal connection is typically written as 
\begin{align}\label{eq:defCurvatureStandard} 
  \mathcal{B}^d_{\alpha\beta}(s)
  :=
  \frac{\partial \mathcal{A}^d_\alpha}
       {\partial s^\beta} -
  \frac{\partial \mathcal{A}^d_\beta}
       {\partial s^\alpha} +
  C^d_{ab} 
       \mathcal{A}^a_\alpha 
       \mathcal{A}^b_\beta 
\end{align}
for $d$, $\alpha$, $\beta \in \{1,\dots,n_J\}$ with
$\mathcal{A}^i_j$ denoting the entry $(i,j)$ of $\mathcal{A}$ 
and $C^a_{bc}$ the structure constants of the Lie group, 
representing the Lie bracket operation that, for $\SE(3)$, is the 6D vector cross product appearing in \eqref{eq:curvatureconnectionSE3}. 
Therefore, the right hand side of \eqref{eq:curvatureconnectionSE3} is equivalent to \eqref{eq:defCurvatureStandard}. 
We find \eqref{eq:curvatureconnectionSE3}, however,
more accessible than \eqref{eq:defCurvatureStandard}
in particular to researchers in multibody dynamics 
employing spatial vector notation \cite{featherstone2008handbook},
not acquainted with differential geometry.
An alternative proof of Proposition~\ref{prop:integrability}
could be obtained by showing that when
\eqref{eq:defCurvatureStandard} is identically zero (flatness of the connection) this implies the existence of a function $\ls^B F$
such that \eqref{eq:trivializedDerivatives} holds. 
We are not aware, though, of an accessible source where 
this is clearly stated as in the proof of the proposition above
and this why we though worth presenting a proof that
requires a basic knowledge of the velocity kinematics of $\SE(3)$ and standard results of calculus (Schwartz's theorem)
to be understood. 
The closest we can get is Chapter II, Section 9,
of the classical text \cite{KoNu96B_FoundDiffGeoVol1}, which clearly requires a deep knowledge on differential geometry to be fully understood.
For sake of completeness, we also mention 
that \eqref{eq:schwartz} could be replaced with the second covariant derivative of $F$ with respect to the $(0)$ Cartan-Schouten connection \cite{Saccon2013},
which is known to be symmetric, to obtain the following expression, similar to \eqref{eq:secondDerF} but now coordinate independent,
\begin{align}
  \left( 
  \frac{\partial \mathcal{A}_i}{\partial s_j}  (s)
  +
  \frac 12 \, 
  \mathcal{A}_i(s) \times \mathcal{A}_j(s)
  \right )^\wedge F(s) \,,  
\end{align}
which leads then again to \eqref{eq:curvatureconnectionSE3}.
Principal connections have been also employed to
study nonholonomic locomotion, as
the nonholomic constraint of a robot can be written in a form equivalent to \eqref{eq:AXB_vpAsdots}: see, e.g., \cite{Ostrowski1998}. We hope 
the presentation given in this paper 
will help also accessing that literature.

\subsection{The link between locked and average velocity}

In this subsection, we elaborate further on the 
remark given in the subsection~\ref{sec:centroidalFrame}
showing that, when choosing $\ls^A o_C = \ls^A  p_{com}$, 
the centroidal kinematic is simply given by
\begin{align} \label{eq:dotAoC}
  \ls^A {\dot o}_C 
& = 
  \ls^A {\dot p}_{com} 
\\ \label{eq:dotARC}
  \ls^A {\dot R}_C 
& = 
  \ls^A \omega_{loc}^\wedge \ls^A R_C 
\end{align}
with $\ls^A \omega_{loc}$ given by the angular velocity component of $\ls^A \rmv_{loc}$. 
We first show that, independently 
from where $\ls^A o_C$ is located,
$\ls^G \rmv_{ave}$ and $\ls^A \rmv_{loc}$
satisfy
\begin{align}\label{eq:GVave}
  \begin{bmatrix}
    \ls^A \dot p_{com} \\
    \ls^A \omega_{loc}
  \end{bmatrix} 
& =  
  \ls^G \rmv_{ave}  
  = 
  \ls^G X_A \ls^A \rmv_{loc} .
\end{align}
This shows that the
average angular velocity $\ls^G \omega_{ave}$
coincides with the locked angular velocity $\ls^A \omega_{loc}$.
Proving \eqref{eq:GVave} is obtained by employing the remark in the subsection~\ref{sec:centroidalFrame} regarding the invariance of the coordinates of the CoM and the following lemma.\\[.5ex]
\noindent{\bf Lemma.}
Given the differential equation
\begin{align}
   \ls^A {\dot H}_C & = \ls^A \rmv^\wedge \, \ls^A H_C
\end{align}
assume there is a point $p$ such that its coordinates
$\ls^C p$ with respect to $C$ are constant. Define $G = (p, [A])$
so that $G$ has $p$ as origin and the same orientation of $A$. 
Then, the velocity $\ls^A \rmv$ written with respect to
$G$ equals
\begin{align}
  \ls^G \rmv 
& = 
  \ls^G X_A \ls^A \rmv    
  = 
  \begin{bmatrix}
    \ls^A \dot p \\
    \ls^A \omega
  \end{bmatrix}
\end{align}
where $\ls^A \omega$ denotes the angular velocity component
of $\ls^A \rmv$. \hfill$\blacksquare$\\[.5ex]

Finally, to obtain \eqref{eq:dotAoC}-\eqref{eq:dotARC}, 
we employ \eqref{eq:GVave} to express $\ls^A \rmv_{loc}$
in terms of $\ls^G \rmv_{ave}$ and substitute it
into \eqref{eq:centroidalFrame_ODE},
obtaining
\begin{align}
  \ls^A {\dot o}_C 
& = 
  \ls^A {\dot p}_{com} 
  + \ls^A \omega_{loc}^\wedge (\ls^A o_C -  \ls^A  p_{com})
\\
  \ls^A {\dot R}_C 
& = 
  \ls^A \omega_{loc}^\wedge \ls^A R_C .
\end{align}
where we recall 
$\ls^A H_C = (\ls^A R_C, \ls^A o_C; 0_{1\times 3}, 1) \in \SE(3)$. As we have selected $\ls^A o_C = \ls^A  p_{com}$,
the result follows.

\section{A numerical example}

In this section, a simple example 
to illustrate the use of
the integrability condition 
\eqref{eq:curvatureconnectionSE3} is given.
We consider a mechanism with two 
internal DOFs. This is the minimal
number of DOFs to observe the nonintegrability,
because, for one DOF, \eqref{eq:curvatureconnectionSE3}
is always trivially satisfied.

We numerically integrate \eqref{eq:centroidalFrame_ODE},
performing a motion that starts and ends at the same 
internal joint configuration. 
The base link will not, in general, return to the original pose.
The centroidal frame will always return 
to the original orientation relative to the base link, 
if and only if \eqref{eq:curvatureconnectionSE3} holds. 
In both cases, 
as explained in the Remark of Section~\ref{sec:centroidalFrame},
its CoM will return to its original position.

An illustration of the mechanism is given in 
Figure~\ref{fig:clock}.
The mechanism is composed by three rigid bodies: 
a free-floating base link (yellow) 
and two distal links (cyan and magenta). 

\begin{figure}
\centering
\includegraphics[width=.3\textwidth]{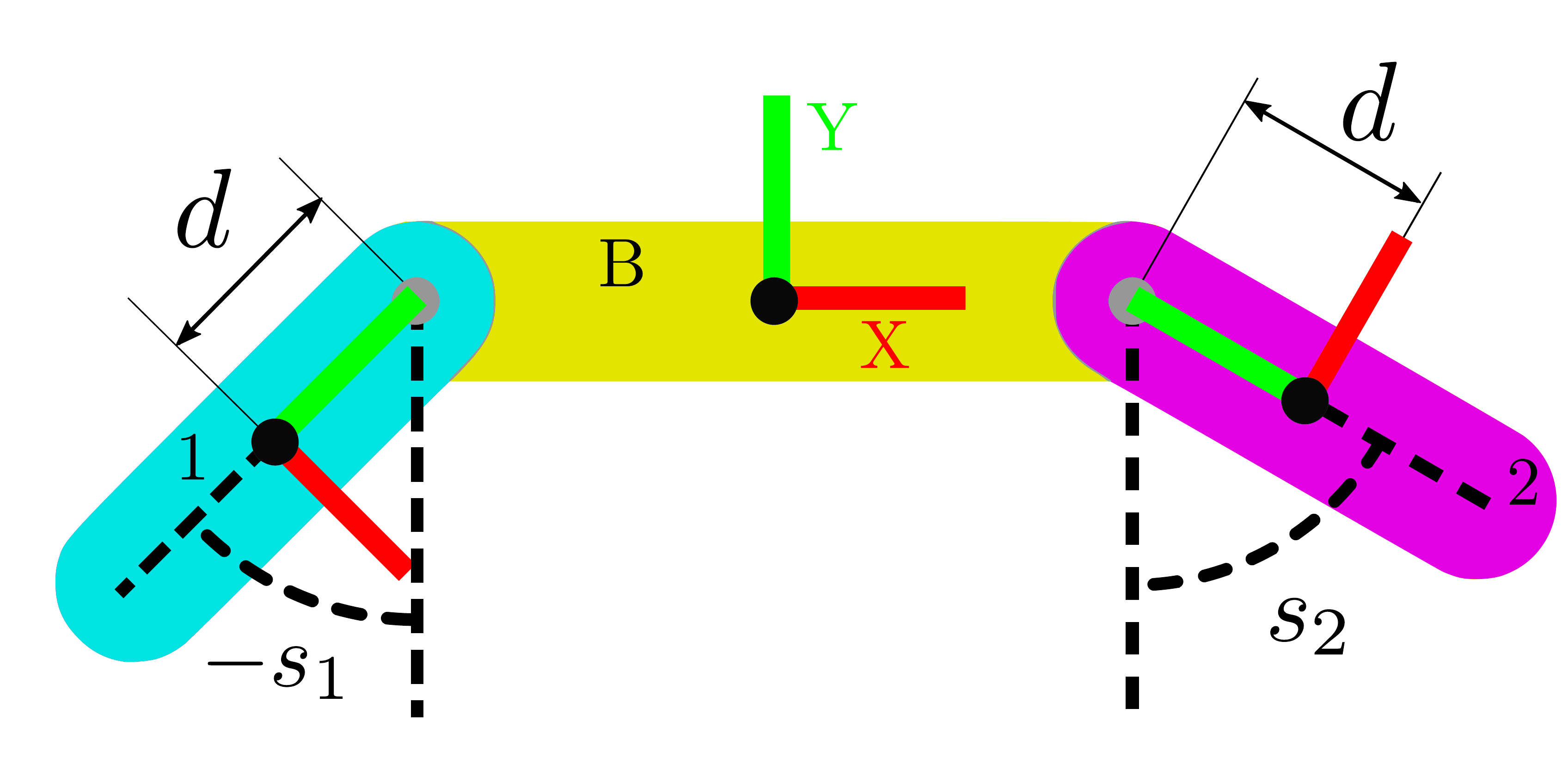}
\caption{The free-floating three link model. 
In this figure, $s_1$ and $s_2$
represent the relative orientations
of the two distal links with respect to the base. See main text
for a full description of the figure. 
}
\label{fig:clock}
\end{figure}

The distal links are connected directly to the base
via two independently actuated revolute joints. 
For both links, the offset between their center of mass and 
joint axis is identical and denoted 
with $d$. 
To each body we firmly attach three coordinate frames,
indicated as $B$, $1$ and $2$ in the figure,
each centered at the corresponding body's CoM. 
The base link mass is $1$ kg. 
Each distal link mass is also $1$ kg.
For the base link, the rotational inertia 
(about the axis passing through the CoM 
and orthogonal to the base link face) is 
$4$ kg ${\textrm m}^2$. For distal links,
the inertia is $1$ kg ${\textrm m}^2$. The rotational inertia 
with respect to the other directions are non influential 
(we are considering a planar mechanism) 
and can be assigned arbitrarily finding the same result
provided below. 

We verify \eqref{eq:curvatureconnectionSE3}
for two different values of $d$: namely, 
for $d = 1$ and $d = 0$. For $d=1$, 
$\mathcal{B}_{12} = -\mathcal{B}_{21}$ is equal, up to a division by
the factor ${\left(2\, C_{1-2} + 6\, S_1 - 6\, S_2 - 28\right)}^2$, to 
\begin{align*} 
  \begin{bmatrix}
  \scriptstyle 
     2 ( C _1 + C\!_2) \, 
     \left(4\, C\!_1 + 4\, C\!_2 - 3\, C\!_1\, S\!_2 + 3\, C\!_2\, S\!_1\right) 
  \\    
  \scriptstyle 
2\, C\!_1\, \left(4\, S\!_1 + 4\, S\!_2 - 3\, S\!_1\, S\!_2 - 3\, S\!_2 S\!_2 \right) + 2\, C\!_2\, \left(4\, S\!_1 + 4\, S\!_2 + 3\, S\!_1\, S\!_2 + 3\, S\!_1 S\!_1 \right)
  \\
    0
  \\
    0
  \\
    0
  \\
    \scriptstyle 
  - 18\, S\!_{1-2} - 24\, C\!_1 - 24\, C\!_2
  \end{bmatrix}
\end{align*}
where $S_1 := \sin(s_1)$,
$S_2 := \sin(s_2)$, $C_1 := \cos(s_1)$,
$C_2 := \cos(s_2)$, $S_{1-2} := \sin(s_1-s_2)$,
and $C_{1-2} := \cos(s_1-s_2)$.
For $d=0$, $\mathcal{B}_{12} = \mathcal{B}_{21} \equiv 0$.
Details of the straightforward but tedious computations 
are not provided for space limitations.

The conclusion is that, just for $d = 0$, the
integration of the average angular velocity 
will {\em always} produce a centroidal frame whose
orientation is only a function of the
internal joint displacements.
This is confirmed by the animation snapshots 
given in Figure~\ref{fig:snapshots} 
corresponding to 
two simulations for different values of $d$. 
The complete animation is available as a multimedia attachment to this paper.
For both cases, the joints follow
the sinusoidal trajectories given by
\begin{align*}
s_1(t) = \frac{3\pi}{2} \left(\cos{\left(\frac{2\pi}{T}t\right)} -1 \right), \ 
s_2(t) = \frac{\pi}{2} \sin{\left(\frac{2\pi}{T}t\right)},
\end{align*}
starting and ending in a mechanism configuration with the distal links in a vertical position. 
These results are independent of the 
particular initial pose and velocity of the
base and therefore, in Figure~\ref{fig:snapshots}, 
only the relative pose of the centroidal frame 
with respect the base link is shown. 

\begin{figure*} 
\centering
\begin{tikzpicture}
\node[anchor=south west,inner sep=0] (image) at (0,0,0) 
{\includegraphics[width=.16\linewidth]{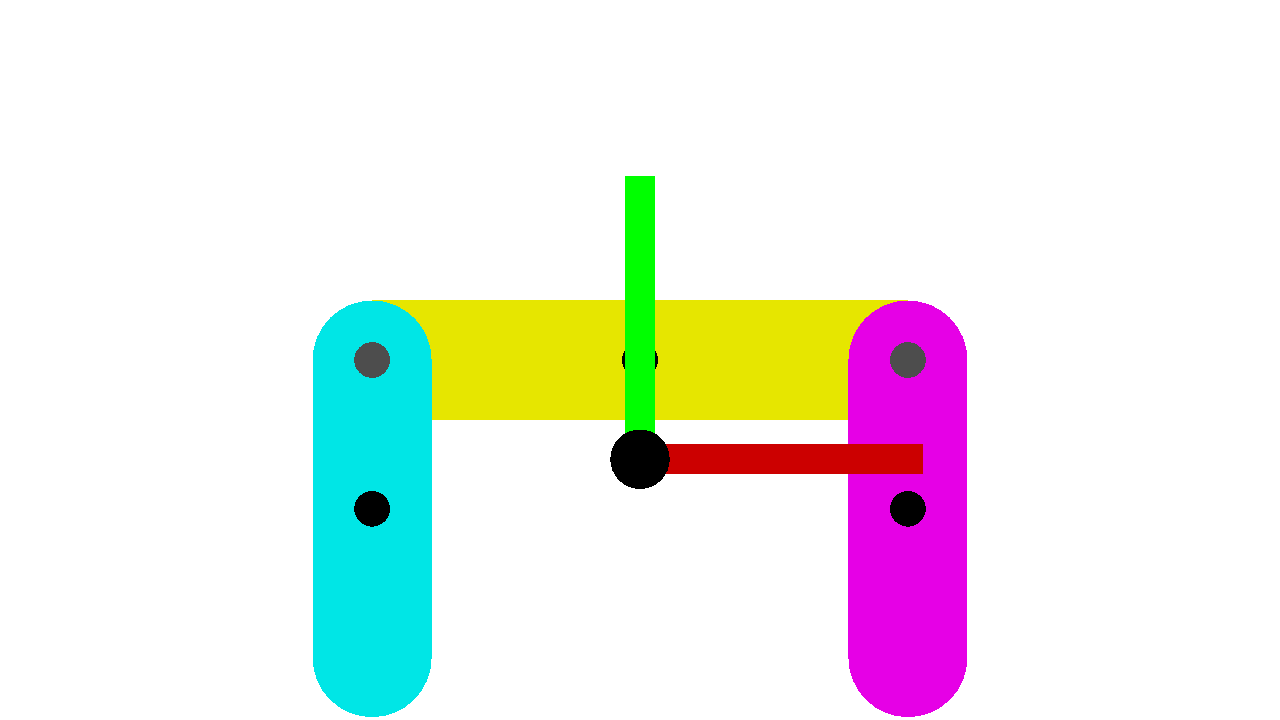}};
\draw [<-,thick,cyan,line width=1pt] (0.50,0.20) arc (-180:-90:0.3);
\draw [->,thick,magenta,line width=1pt] (2.12,-0.07) arc (-90:-00:0.3);
\end{tikzpicture}
\hfill
\begin{tikzpicture}
\node[anchor=south west,inner sep=0] (image) at (0,0,0) 
{\includegraphics[width=.16\linewidth]{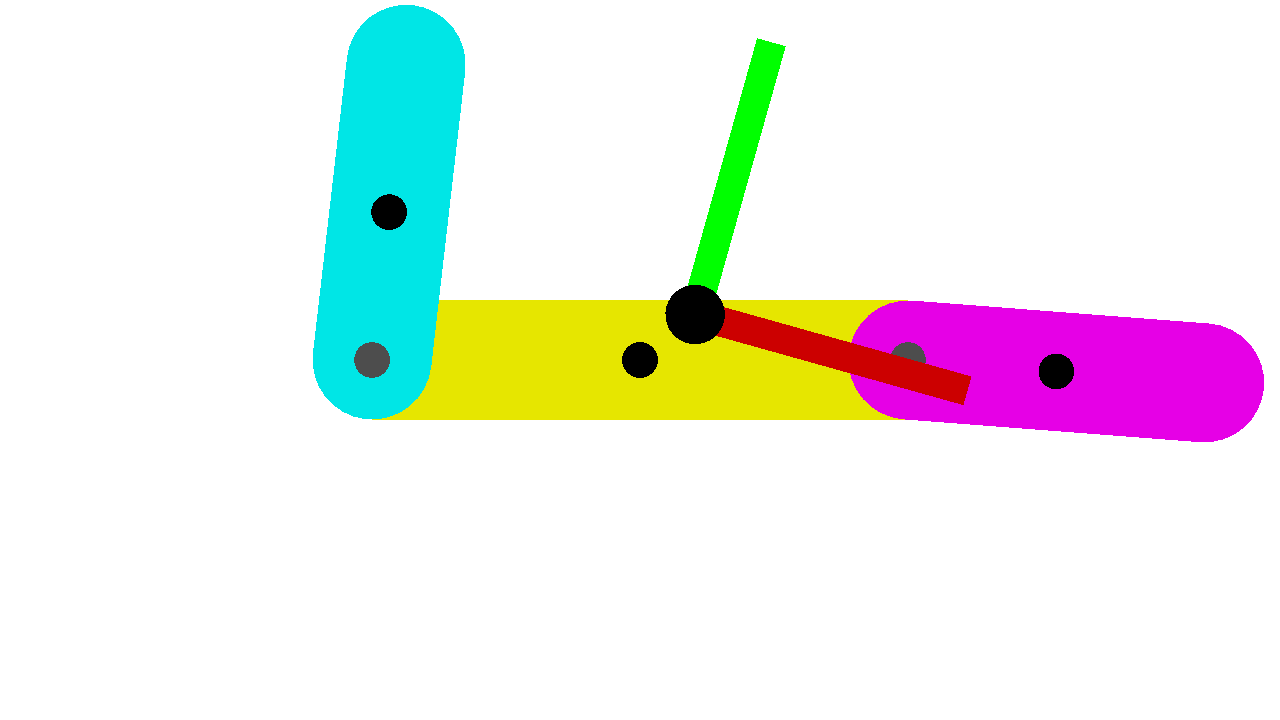}};
\draw [->,thick,cyan,line width=1pt] (0.66,1.7) arc (135:45:0.3);
\draw [<-,thick,magenta,line width=1pt] (2.7,0.5) arc (-90:-00:0.3);
\end{tikzpicture}
\hfill
\begin{tikzpicture}
\node[anchor=south west,inner sep=0] (image) at (0,0,0) 
{\includegraphics[width=.16\linewidth]{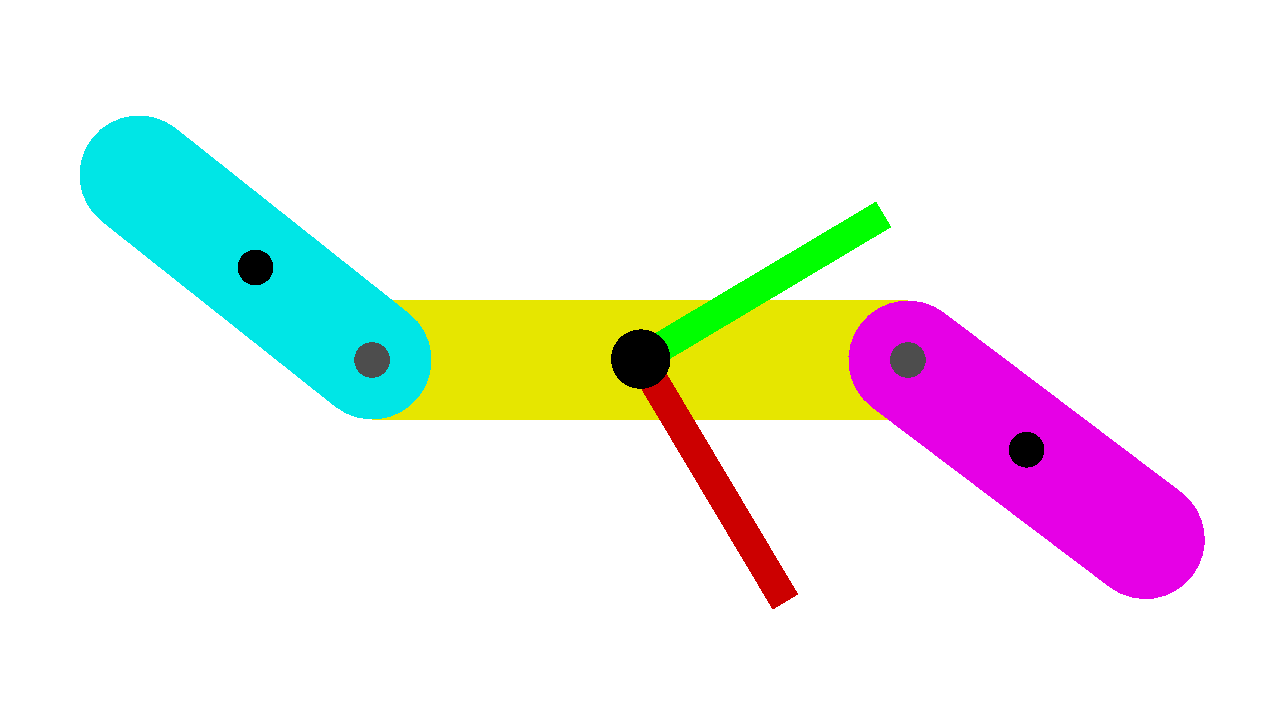}};
\draw [->,thick,cyan,line width=1pt] (0.1,1.4) arc (150:60:0.3);
\draw [<-,thick,magenta,line width=1pt] (2.57,0.15) arc (-90:-00:0.3);
\end{tikzpicture}
\hfill
\begin{tikzpicture}
\node[anchor=south west,inner sep=0] (image) at (0,0,0) 
{\includegraphics[width=.16\linewidth]{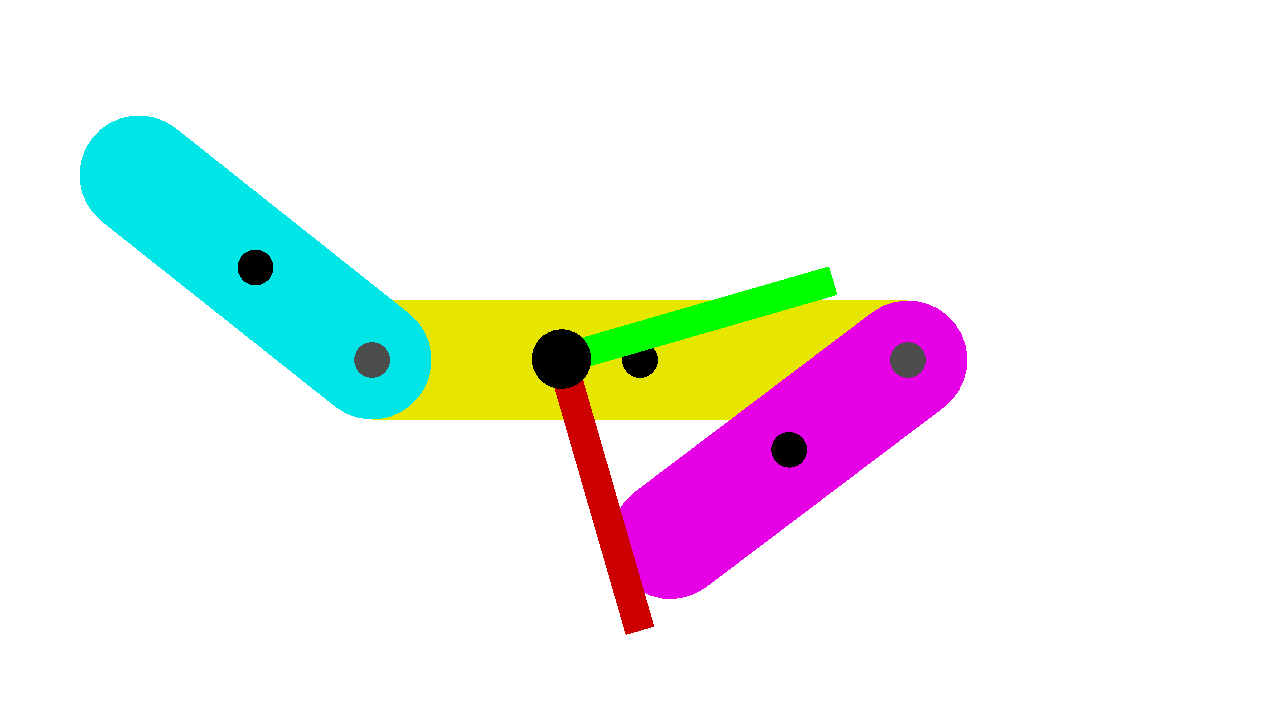}};
\draw [<-,thick,cyan,line width=1pt] (0.1,1.4) arc (150:60:0.3);
\draw [<-,thick,magenta,line width=1pt] (1.37,0.2) arc (-145:-45:0.3);
\end{tikzpicture}
\hfill
\begin{tikzpicture}
\node[anchor=south west,inner sep=0] (image) at (0,0,0) 
{\includegraphics[width=.16\linewidth]{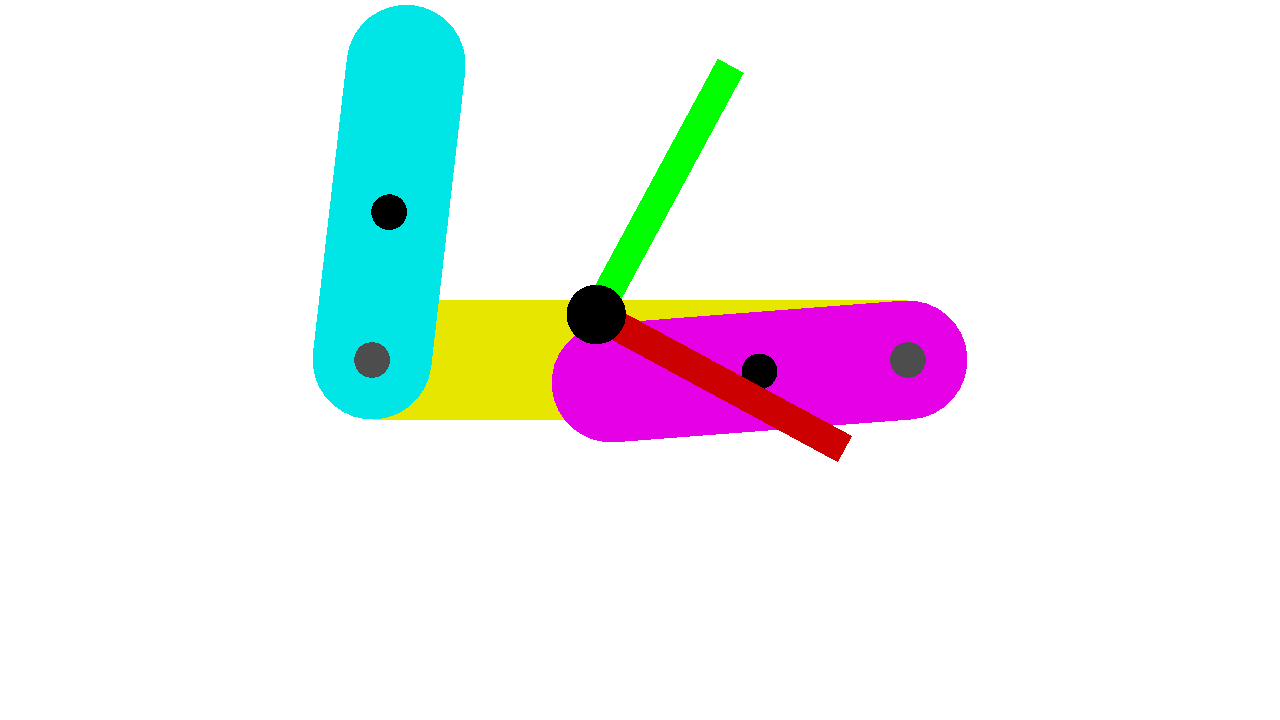}};
\draw [<-,thick,cyan,line width=1pt] (0.6,1.6) arc (150:60:0.3);
\draw [->,thick,magenta,line width=1pt] (1.2,0.7) arc (-170:-80:0.3);
\end{tikzpicture}
\hfill
\begin{tikzpicture}
\node[anchor=south west,inner sep=0] (image) at (0,0,0) 
{\includegraphics[width=.16\linewidth]{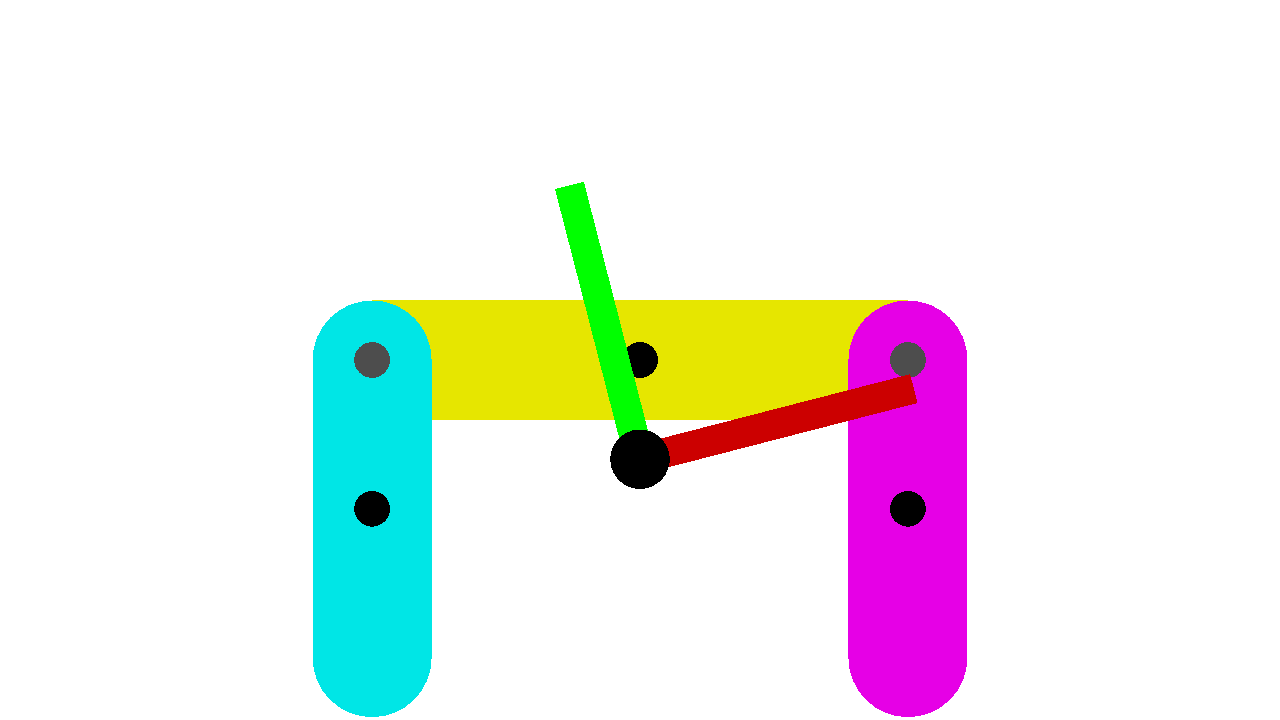}};
\end{tikzpicture}
\\~\\
\minipage{0.16\textwidth}
\begin{tikzpicture}
\node[anchor=south west,inner sep=0] (image) at (0,0,0) 
  {\includegraphics[width=\linewidth]{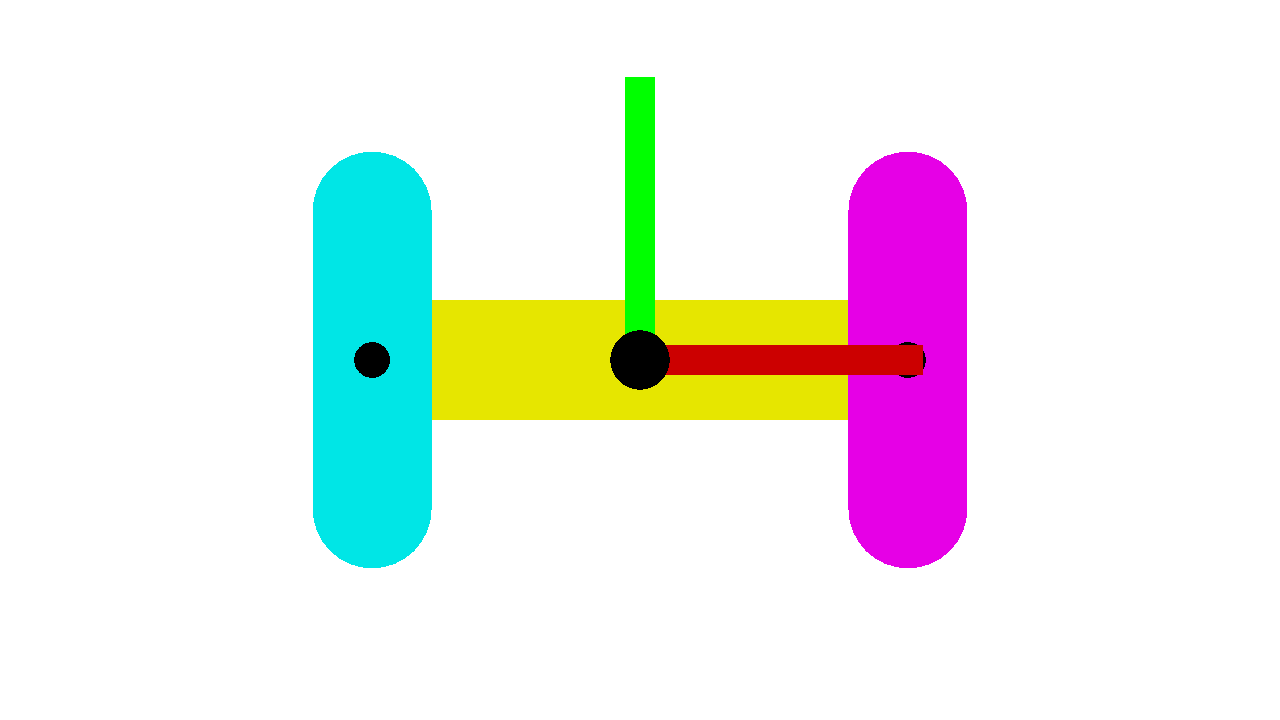}};
\draw [<-,thick,cyan,line width=1pt] (0.50,0.55) arc (-180:-90:0.3);
\draw [->,thick,magenta,line width=1pt] (2.12,0.26) arc (-90:-00:0.3);
  \end{tikzpicture}
  \\[1ex]  
  \centering 
  $t = 0$\ s
\endminipage\hfill
\minipage{0.16\textwidth}
\begin{tikzpicture}
\node[anchor=south west,inner sep=0] (image) at (0,0,0) 
  {\includegraphics[width=\linewidth]{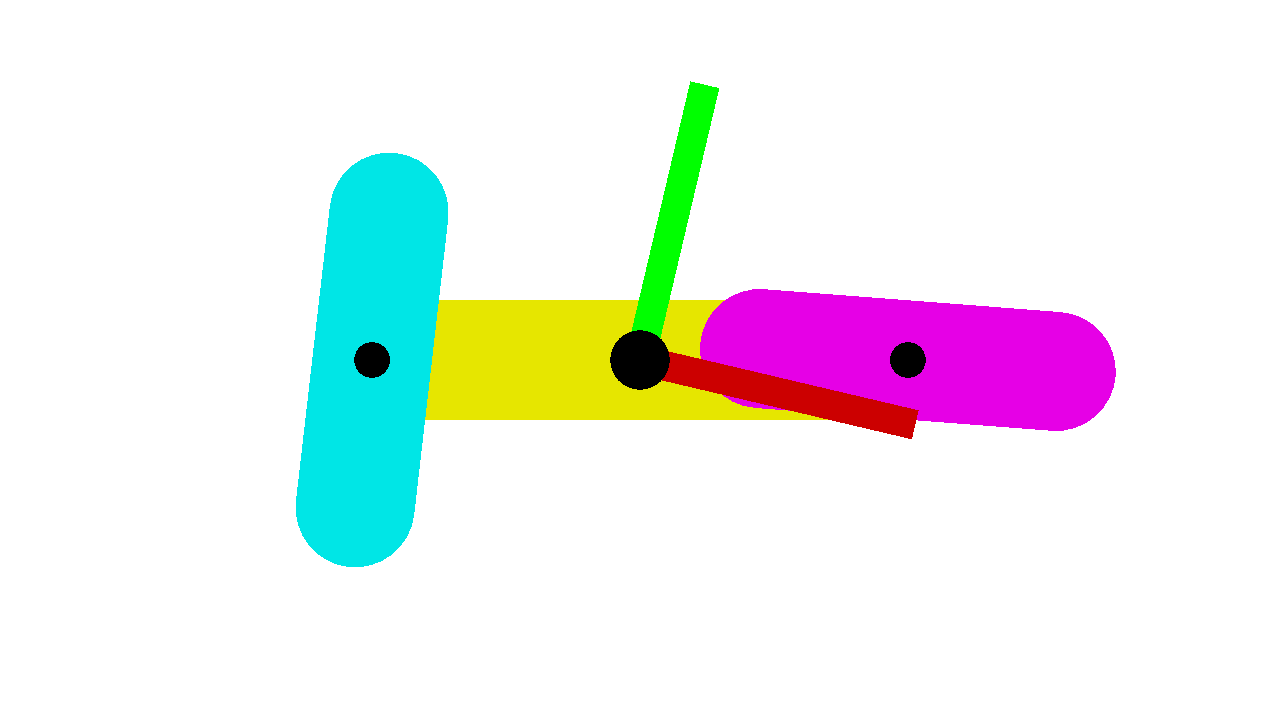}};
  \draw [->,thick,cyan,line width=1pt] (0.66,1.35) arc (135:45:0.3);
\draw [->,thick,magenta,line width=1pt] (2.4,0.6) arc (-90:-00:0.3);
  \end{tikzpicture}
  \\[1ex]  
  \centering 
  $t = 2$\ s
\endminipage\hfill
\minipage{0.16\textwidth}
\begin{tikzpicture}
\node[anchor=south west,inner sep=0] (image) at (0,0,0) 
  {\includegraphics[width=\linewidth]{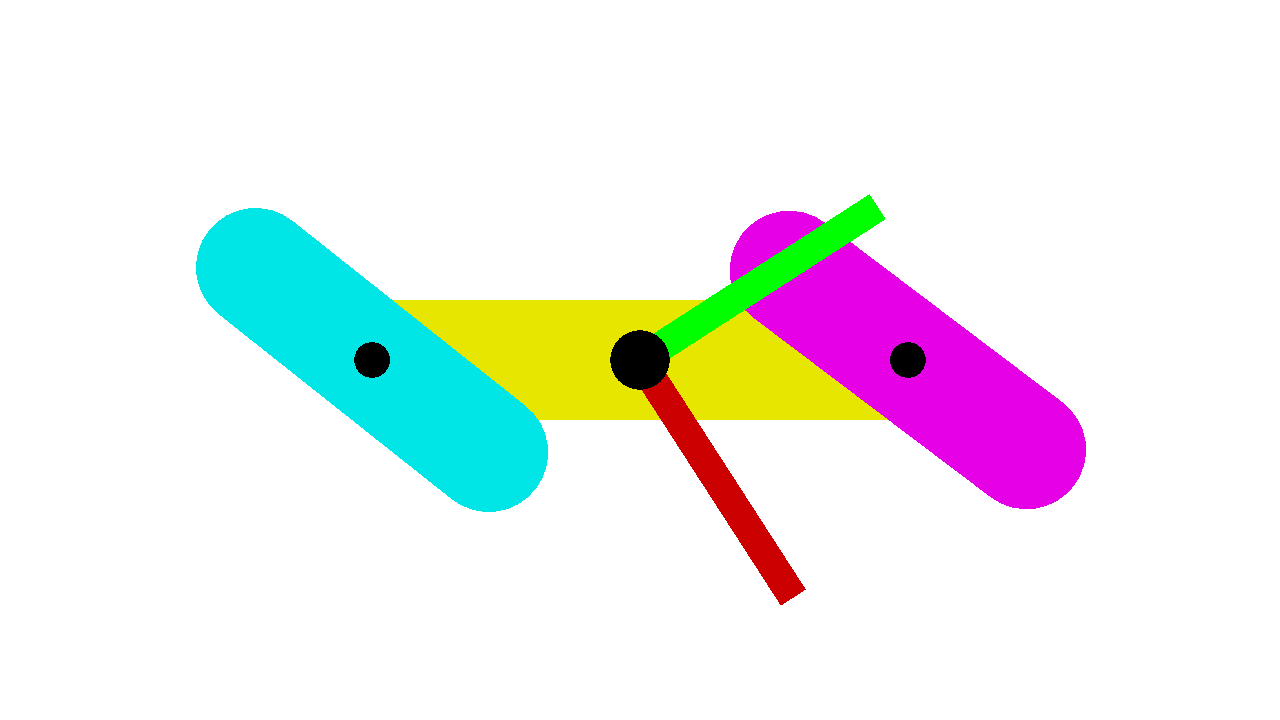}};
  \draw [->,thick,cyan,line width=1pt] (0.3,1.15) arc (150:60:0.3);
\draw [<-,thick,magenta,line width=1pt] (2.3,0.35) arc (-90:-00:0.3);
  \end{tikzpicture}
  \\[1ex] 
  \centering 
  $t = 4$\ s
\endminipage\hfill
\minipage{0.16\textwidth}
\begin{tikzpicture}
\node[anchor=south west,inner sep=0] (image) at (0,0,0) 
  {\includegraphics[width=\linewidth]{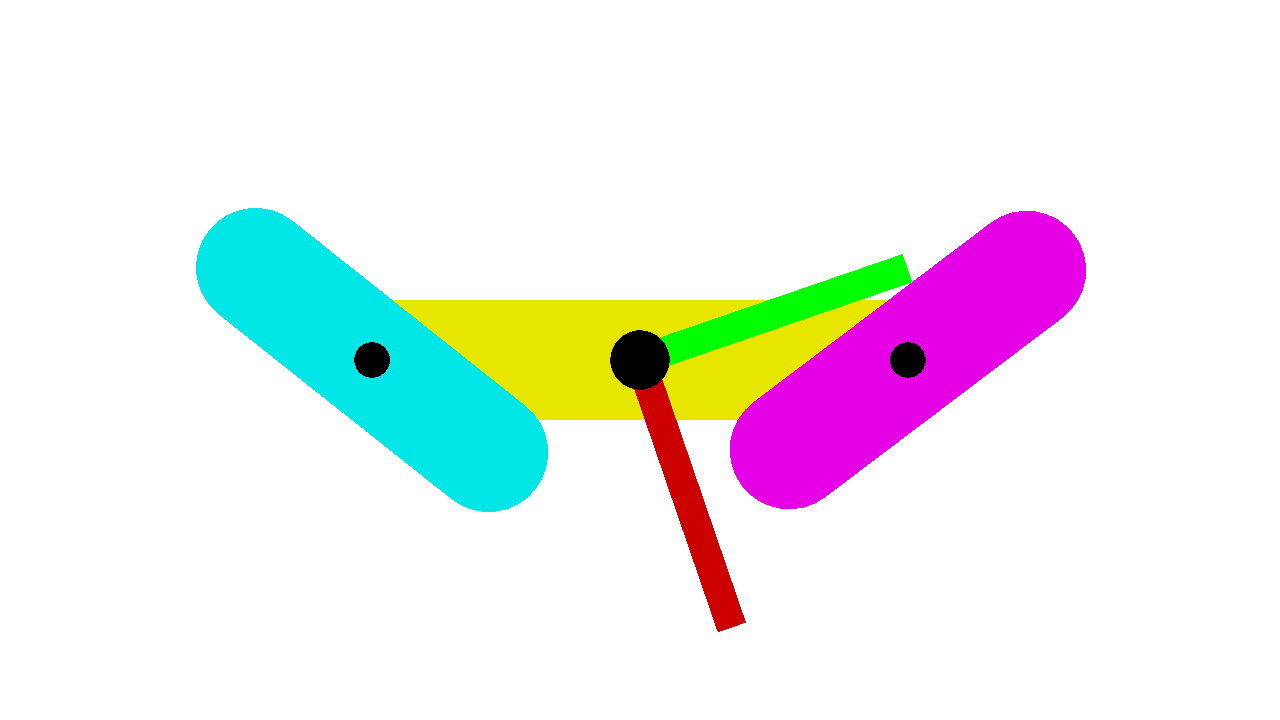}};
  \draw [<-,thick,cyan,line width=1pt] (0.3,1.15) arc (150:60:0.3);
\draw [<-,thick,magenta,line width=1pt] (1.7,0.45) arc (-145:-45:0.3);
  \end{tikzpicture}
  \\[1ex]  
  \centering 
  $t = 6$\ s
\endminipage\hfill
\minipage{0.16\textwidth}
\begin{tikzpicture}
\node[anchor=south west,inner sep=0] (image) at (0,0,0) 
  {\includegraphics[width=\linewidth]{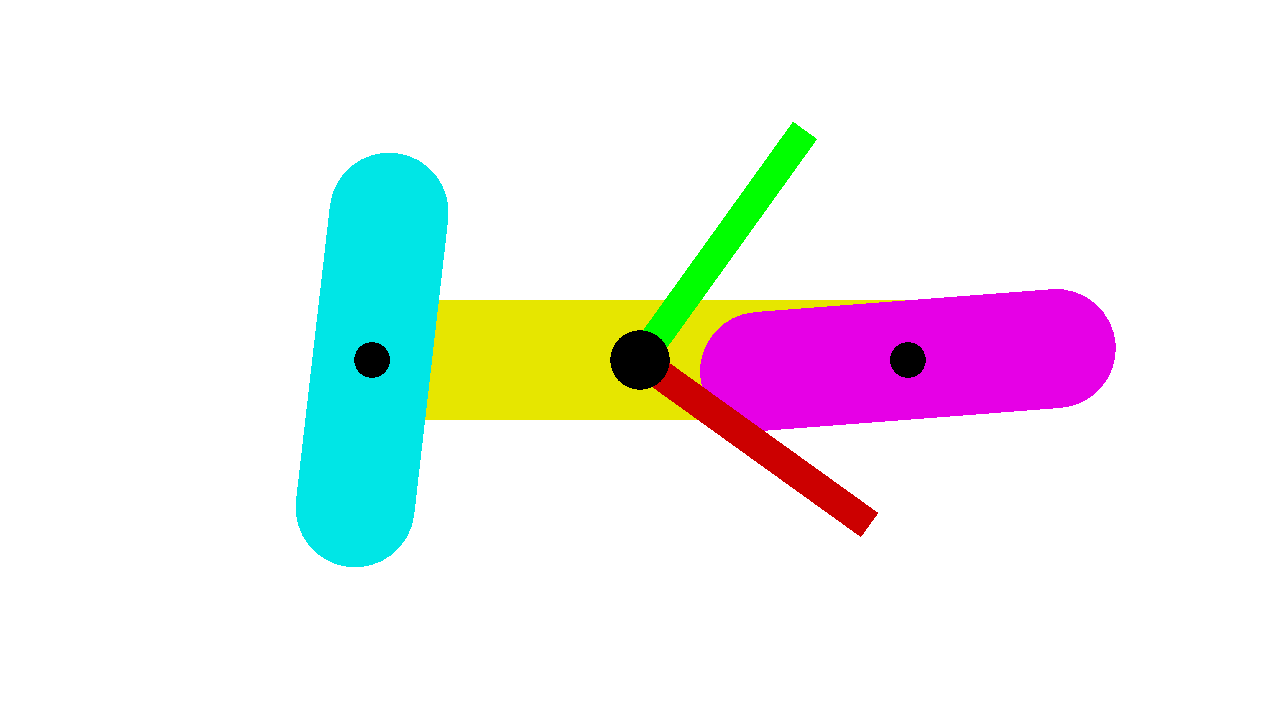}};
  \draw [<-,thick,cyan,line width=1pt] (0.6,1.3) arc (150:60:0.3);
\draw [->,thick,magenta,line width=1pt] (1.4,0.7) arc (-170:-80:0.3);
  \end{tikzpicture}
  \\[1ex] 
  \centering 
  $t = 8$\ s
\endminipage\hfill
\minipage{0.16\textwidth}
\begin{tikzpicture}
\node[anchor=south west,inner sep=0] (image) at (0,0,0) 
  {\includegraphics[width=\linewidth]{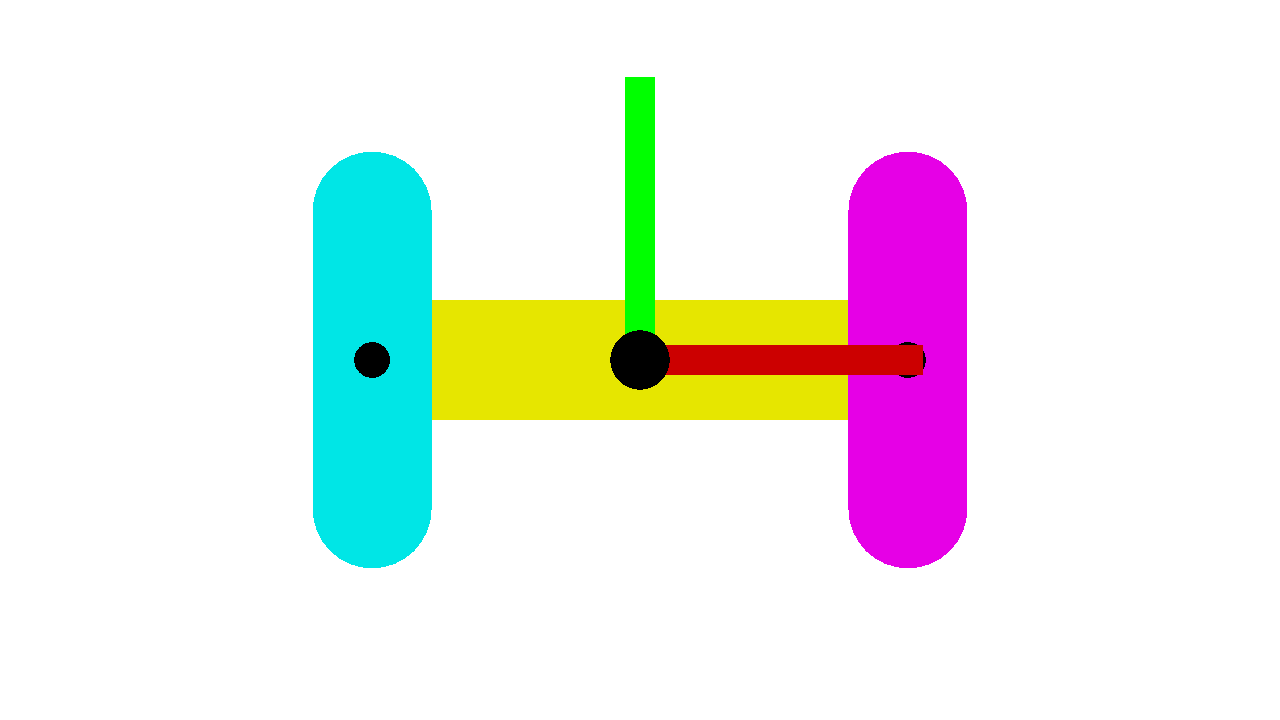}};
  \end{tikzpicture}
  \\[1ex] 
  \centering 
  $t = 10$\ s
\endminipage
\caption{Evolution of the centroidal frame. 
Nonintegrable case with $d = 1$ (first row).
Integrable case with $d = 0$ (second row) }
\label{fig:snapshots}
\end{figure*}
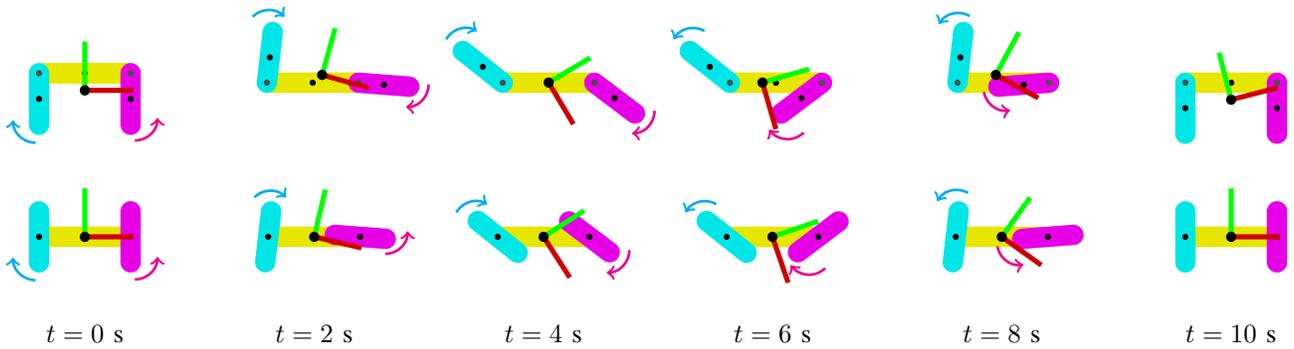

\section{Conclusion and discussion}\label{sec:conclusion}

In this paper we have established a clear 
link between the concept of average
angular velocity found in the robotics and multibody dynamics literature and the 
concept of locked velocity from the geometric mechanics literature. We have 
provided an accessible definition
and proof on the key algebraic condition
that can be used to establish 
when the centroidal frame (in particular, its orientation) depends only on the current robot
configuration and not on its time evolution.
For systems with a small number of links, the condition could be checked symbolically and, for more complex mechanism, we expect efficient algorithms could be easily developed
as just the differentiation of the matrices $\mathbb{L}$ and $\mathbb{A}$, typically involved in the computation of the Coriolis' matrix associated to the dynamics, is required.
Computationally, our experience suggests to compute the locked velocity with respect the frame whose origin is the CoM and orientation given by the base link frame. The advantage of using this frame is that the locked inertia matrix becomes diagonal, 
still depending only on the shape variables.

The centroidal frame is always integrable
when dealing with a system possessing only one internal degree of freedom. Note how this result
would remain valid for a mechanism where virtual constraints are employed to ensure that this
happens, making the internal degrees of freedom algebraically related to a single variable that acts as the only internal degree of freedom (see, e.g., the concept of gate variable in \cite{Grizzle2001} and the work that stemmed from it in employing virtual constraints in the context of robot locomotion \cite{Grizzle2014}). Our presentation might therefore help in taking a different perspective to those results. 

We have shown that mechanisms with
two internal degrees of freedom can also possess an integrable centroidal frame, but this is not guaranteed as common experience and previous investigations (e.g., the so called falling cat problem)
have already shown. It might be interesting to understand if there are some rules in constructing a (nontrivial) mechanism such that the integrability is satisfied. The centroidal frame for those mechanism might provide an interesting natural output to use in controlling the gross motion of the systems both in position and in orientation. Even when dealing with a mechanism where the centroidal frame is not integrable, one idea could be to try to assign the external wrenches so as to guarantee the integrability condition: one of the motivations to continue investigating the centroidal dynamics even further.

\appendix 

\subsection{The momentum map and the total momentum}
\label{sec:AppendixA}

Noether's theorem can be employed to conclude
that, when no gravity and external forces are applied, the momentum $\ls_A J$ given by \eqref{eq:momentum} is constant. 
In the context of geometric mechanics, \eqref{eq:momentum} is the momentum map $\mathcal{J}$ defined as 
\begin{align}\label{eq:defmomentummap}
  \left< \mathcal{J}( q , \dot q ), \xi \right> 
  & = 
  \mathbb{F} L ( q, \dot q ) \cdot \xi_Q (q)   
\end{align}
where $\mathbb{F} L ( q, \dot q ) \cdot z := \lim_{t\rightarrow 0}
(L(q , \dot q + t z ) - L(q, \dot q)) / t$ denotes the fiber derivative of 
the Lagrangian in the direction $z \in T_q Q$ and $\xi_Q (q)$  
the infinitesimal generator of the group action formally defined as 
$\xi_Q (q) = d/dt|_{t=0} \Phi_{\exp(t \xi)}(q)$.  
The function $\Phi_g$ is the action of the symmetry group 
on the configuration space: in the context of this paper,
$g \in \SE(3)$, the configuration space $Q = SE(3) \times \R^{n_J}$,
and the group action $\Phi_g: Q \rightarrow Q$ is simply given by
\begin{align}\label{eq:groupaction}
  \Phi_g(H,\rms) = (g H, \rms) ,
\end{align}
corresponding to a rigid transformation of the entire robot 
according to $g$ that leaves
invariant the shape $\rms$. The infinitesimal generator
associated to \eqref{eq:groupaction} is therefore
\begin{align}
  \xi_Q(q) = (\xi^\wedge H, 0) \in T_{(H, \rms)} Q   
\end{align}
and after straightforward computations one gets that 
the momentum map equals \eqref{eq:momentum} 
as the right hand side of \eqref{eq:defmomentummap} is 
\begin{align}\label{eq:FLxiQ}
\mathbb{F} L ( q, \dot q ) \cdot \xi_Q (q)  =
\begin{bmatrix}
  \rmv \\
  \dot \rms
\end{bmatrix}^T
\begin{bmatrix}
  \mathbb{L}(\rms)   & \mathbb{A}(\rms) \\
  \mathbb{A}^T(\rms) & \mathbb{S}(\rms) 
\end{bmatrix}
\begin{bmatrix}
  \ls^B X_A \, \xi \\
  0
\end{bmatrix} .
\end{align}
For the reader that is familiar with Lie group theory, note that, in \eqref{eq:FLxiQ}, $\ls^B X_A = \Ad_{H^{-1}}$. For more details on momentum maps and related concepts, we refer the interested reader to \cite[Chapter 11]{marsden2013introduction} and \cite{Bloch2003}.

\subsection{The center of mass is always a fixed point}
\label{sec:AppendixB}

In this appendix, we prove that the center 
of mass is always a fixed with respect to $\ls^A H_C$ 
obtained by the time integration of \eqref{eq:centroidalFrame_ODE}.
Requiting that the CoM to be fixed with respect to
the frame $C$ is equivalent to ask that
\begin{align}\label{eq:dotAbarp}
  \ls^A \dot {\bar p}_{com} 
  & {=} 
  \ls^A \rmv_{loc}^\wedge \ls^A \bar p_{com}  
\end{align}
where $\ls^A{\bar p}_{com}$ are the homogeneous coordinates of
$p_{com}$ with respect to $A$ obtained by appending $1$
to the standard coordinates $\ls^A p_{com}$, i.e.,
$\ls^A{\bar p}_{com} := (\ls^A p_{com}; 1)$ where ; 
denotes row concatenation. The proof of this fact
derives from a straightforward manipulation
of the expression of the time derivative 
of the identity $\ls^A {\bar p}_{com} = \ls^A H_C \ls^C {\bar p}_{com}$
assuming $\ls^C p_{com}$ to be a constant.
The right hand side of 
\eqref{eq:dotAbarp} can be expressed with respect
to frame the frame $G = (p_{com}, [A])$ obtaining
the equivalent condition
\begin{align}\label{eq:dotAbarp_v2}
  \ls^A \dot {\bar p}_{com}  
  & {=} 
  \ls^A H_G \ls^G \rmv_{loc}^\wedge \ls^G \bar p_{com}
\end{align}
where $\ls^G \bar p_{com} = (0,0,0,1)^T$. 
Equation \eqref{eq:dotAbarp_v2} is then equivalent to
$\ls^A {\dot p}_{com} = \ls^A R_G \ls^G v_{loc}$
and to 
$\ls^A {\dot p}_{com} = \ls^G v_{loc} $, since $\ls^A R_G = I$. 
This last condition is always true deriving 
directly from the fact that the momentum map 
$\mathcal{J}$ expressed in the $G$ frame is given by
\begin{align*}
  \ls_G \mathcal{J}(H,\rms,\rmv,\dot\rms) 
  & = 
  \ls_G X^B \ls_B \mathbb{L}_B \ls^B \rmv_{loc}  
  \\
  & =
  \ls_G \mathbb{L}_G  \ls^G \rmv_{loc}
\end{align*}
where $\ls_G \mathbb{L}_G$ is block diagonal with first block on the diagonal equal to $m I_{3\times 3}$ with $m$ the total mass and 
that the linear momentum component of $\ls_G \mathcal{J}$ is 
necessarily $m \ls^A {\dot p}_{com}$.

\bibliographystyle{IEEEtran}
\bibliography{centroidal} 

\end{document}